\documentclass[twoside]{article}

\usepackage[accepted]{aistats2020}
\usepackage[round]{natbib}

\usepackage{amsmath,amssymb,amsthm,mathtools}
\usepackage{algpseudocode,algorithm}

\newtheorem{theorem}{Theorem}[section]
\newtheorem{definition}[theorem]{Definition}
\newtheorem{assumption}[theorem]{Assumption}
\newtheorem{lemma}[theorem]{Lemma}
\newtheorem{remark}[theorem]{Remark}

\newcommand{\X}{\mathcal{X}}

\begin{document}

%
\runningtitle{Sample Complexity of Estimating the Policy Gradient}

%

\twocolumn[

\aistatstitle{Sample Complexity of Estimating the Policy Gradient \\ for Nearly Deterministic Dynamical Systems}

\aistatsauthor{ Osbert Bastani }

\aistatsaddress{ University of Pennsylvania, USA } ]

\begin{abstract}
Reinforcement learning is a promising approach to learning robotics controllers. It has recently been shown that algorithms based on finite-difference estimates of the policy gradient are competitive with algorithms based on the policy gradient theorem. We propose a theoretical framework for understanding this phenomenon. Our key insight is that many dynamical systems (especially those of interest in robotics control tasks) are \emph{nearly deterministic}---i.e., they can be modeled as a deterministic system with a small stochastic perturbation. We show that for such systems, finite-difference estimates of the policy gradient can have substantially lower variance than estimates based on the policy gradient theorem. Finally, we empirically evaluate our insights in an experiment on the inverted pendulum.

\end{abstract}

\section{Introduction}

The policy gradient is the workhorse of modern reinforcement learning. In particular, most state-of-the-art reinforcement learning algorithms aim to learn a control policy $\pi_{\theta}$ by estimating the policy gradient---i.e., the gradient $\nabla_{\theta}J(\theta)$ of the expected cumulative reward $J(\theta)$ with respect to the parameters $\theta$ of the control policy---in one of two ways: (i) numerically, e.g., using a finite-difference approximation~\citep{kober2013reinforcement,mania2018simple}, or (ii) by using the policy gradient theorem~\citep{sutton2000policy} to construct estimates~\citep{silver2014deterministic,schulman2015trust,schulman2015high,schulman2017proximal}. However, there has been little work on theoretically understanding the tradeoffs between these two approaches, and our work aims to help fill this gap.

We are interested in applications to robotics control, which typically have continuous state and action spaces~\citep{collins2005efficient,abbeel2007application,levine2016end}. For example, reinforcement learning can be used to learn controllers when the dynamics are unknown~\citep{abbeel2007application,ross2012agnostic,akametalu2014reachability,berkenkamp2017safe,johannink2018residual}. Understanding sample complexity is especially important in this application, since the goal is for robots to be able to learn based on real world experience, which can be very costly to obtain. Furthermore, having a theoretical understanding of sample complexity is important for developing safe reinforcement learning algorithms~\citep{akametalu2014reachability,berkenkamp2017safe,dean2018safely}.

We argue that \emph{near determinism} is an important characteristic of dynamical systems relevant to robotics. More precisely, we study settings where the noise in the dynamics is ``small'' (i.e., sub-Gaussian with small constant). This setting captures robotics tasks such as grasping~\citep{andrychowicz2018learning}, quadcopters~\citep{akametalu2014reachability}, walking~\citep{collins2005efficient}, and driving~\citep{montemerlo2008junior}, where the dynamics are primarily deterministic but include small perturbations such as wind, friction, or slippage. We discuss this claim in detail below.

\textbf{Main results.}
In the context of near determinism, we analyze the sample complexity of various algorithms for estimating the policy gradient $\nabla_{\theta}J(\theta)$. We study three algorithms: (i) an algorithm based on finite-differences, (ii) an algorithm based on the policy gradient theorem, and (iii) a model-based algorithm (i.e., it knows the system dynamics) that uses backpropagation to estimate the policy gradient. The model-based algorithm represents the best convergence rate we can hope to achieve using only random samples of the noise. We give details on these algorithms in Section~\ref{sec:alg}.

Our key parameter of interest is the sub-Gaussian parameter $\sigma_{\zeta}$ of the system noise $\zeta$, which is small for nearly deterministic systems. Here, we also consider dependences on the estimation error $\epsilon$ and the dimension $d_{\Theta}$ of the parameter space; we state theorems giving dependences on all parameters in Section~\ref{sec:mainresults}. We prove the following bounds on the sample complexity $n$ (i.e., the number of samples needed to get at most $\epsilon$ error with probability at least $1-\delta$):
{\setlength{\parskip}{0pt}
\begin{itemize}
\setlength{\itemsep}{0pt}
\item For the model-based estimate, $n=\tilde{\Theta}(\sigma_{\zeta}^2/\epsilon^2)$.
\item For the finite-differences estimate, $n=\tilde{\Theta}(\sigma_{\zeta}^2d_{\Theta}/\epsilon^4)$.
\item For the estimate based on the policy gradient theorem, $n=\tilde{O}(1/\epsilon^2)$ and $n=\tilde{\Omega}(1/\epsilon)$.
\end{itemize}
Our key finding is that while both the model-based and finite-difference estimates become small as $\sigma_{\zeta}$ becomes small, the estimate based on the policy gradient theorem does not. Thus, for nearly deterministic dynamical systems, finite-difference algorithms perform significantly better. However, this improvement comes at a price---$n$ depends on $d_{\Theta}$, and furthermore quadratically more samples are needed to get to the same estimation error.}

Finally, we focus on how many samples are needed to estimate the policy gradient on a single step. This understanding is already useful for applications such as safe reinforcement learning. Nevertheless, we discuss how our results connect to the problem of optimizing $J(\theta)$ in Section~\ref{sec:mainresults}.

\textbf{Motivation for near determinism.}
A common approach in robotics is to model the robot dynamics as deterministic~\citep{levinson2011towards,kuindersma2016optimization}. To account for stochasticity, either a stabilizing controller such as a PID controller is used~\citep{levinson2011towards}, or the robot's trajectory is replanned at every step~\citep{kwon1983stabilizing,kuindersma2016optimization}. An alternative approach is to assume that the dynamics are deterministic plus a bounded perturbation at each step, and then use robust control~\citep{akametalu2014reachability}. Both approaches implicitly assume that the deterministic portion of the dynamics are a good approximation of the full dynamics.
In general, most systems that have been successfully studied in reinforcement learning are nearly deterministic, including Atari games~\citep{mnih2015human}, MuJoCo benchmarks~\citep{todorov2012mujoco,levine2013guided}, and simulated grasping tasks~\citep{andrychowicz2018learning}.

More importantly, we believe that it will be challenging to increase the sample efficiency of reinforcement learning in systems where the noise is high. Indeed, our analysis shows that noise can be greatly amplified by the dynamics, so if the noise is large, we believe there is very little hope for sample-efficient reinforcement learning. In these settings, we may need to rely on techniques such as transfer learning~\citep{taylor2009transfer}, meta-learning~\citep{finn2017model}, or learning to plan~\citep{tamar2016value} to achieve low sample complexity.

\textbf{Related work.}
The theoretical work in reinforcement learning algorithms has primarily focused on $Q$-learning~\citep{watkins1992q,kearns2002near,kakade2003sample,jin2018q}, especially for Markov decision processes (MDPs) with finite state and action spaces. There has been some work on understanding the sample complexity of reinforcement learning with function approximation---e.g., for fitted value iteration~\citep{munos2008finite}, for fitted policy iteration~\citep{antos2008learning,lazaric2012finite,farahmand2015classification,farahmand2016regularized}, fitted $Q$-iteration~\citep{tosatto2017boosted}, and the $\text{TD}(0)$ algorithm~\citep{dalal2018finite}. For robotics tasks, where state and action spaces are typically continuous, the most successful approaches are predominantly based on policy gradient estimation~\citep{collins2005efficient,kober2013reinforcement}, for which there has been relatively little work. In this direction, ~\citep{kakade2003sample} has analyzed the sample complexity of algorithms based on the policy gradient theorem, but they do not study the dependence of the sample complexity on the magnitude of the system noise. Furthermore, their work assumes finite state and action spaces and bounded rewards, and they do not consider finite-difference algorithms.

There has been work characterizing a key design choice of finite-difference algorithms---i.e., the distribution of perturbations used to numerically estimate the policy gradient~\citep{roberts2009signal}. They measure the performance of different choices using the signal-to-noise ratio. In contrast, our goal is to understand the sample complexity of different algorithms for nearly deterministic systems.

There has recently been work on understanding the sample complexity of learning controllers; however, they focus on linear dynamical systems, and on different algorithms---e.g., temporal difference learning~\citep{tu2018least} or model-based algorithms~\citep{dean2018regret,tu2018gap}. There has also been work in this setting studying the possibility of reducing variance by controlling the noise in the dynamics~\citep{malik2018derivative}; in the setting we study, we cannot control the noise.

There has been recent work comparing approaches based on exploration in the action space (based on the policy gradient theorem) to exploration in the state space (based on finite difference methods)~\citep{vemula2019contrasting}. Our focus on nearly deterministic systems enables us to obtain qualitatively different insights compared to theirs. In particular, they find that approaches based on finite differences perform better for problems with a long time horizon. However, we analyze a more realistic model, and find that this insight no longer holds. Instead, approaches based on finite differences outperform approaches based on the policy gradient theorem for nearly deterministic systems.

Our analysis differs in three key ways. First, they assume an upper bound $J(\theta)\le J_{\text{max}}$, which is a very strong assumption. Second, their analysis does not model stochastic dynamics. Instead, they assume that $J(\theta)$ is deterministic, but they can only obtain observations $J(\theta)+\zeta$, where $\zeta$ is i.i.d. noise. In contrast, our analysis considers both stochastic dynamics, as well as how noise is propagated through the dynamics. This distinction substantially complicates our analysis, but is necessary for us to understand the implications of near determinism (since we need to understand how the dynamics can amplify noise). Finally, unlike their work, we provide lower bounds for our main results.

\textbf{Connection to optimizing $J(\theta)$.}
Estimating the policy gradient can be used in conjunction with stochastic gradient descent to optimize $J(\theta)$. There is a large body of work on understanding the convergence rate of stochastic gradient descent~\citep{robbins1985stochastic,spall1992multivariate,bottou2008tradeoffs,moulines2011non}, of which policy gradient algorithms are a special case. Indeed, \citep{vemula2019contrasting} uses these techniques to bound the complexity of optimizing $J(\theta)$.

There are several reasons why we focus on understanding the sample complexity of a single gradient step rather than the sample complexity of optimization. First, they rely on the strong assumption that $J(\theta)$ is bounded---i.e., $J(\theta)\le J_{\text{max}}$ for some $J_{\text{max}}\in\mathbb{R}_+$. Second, it would be much more difficult to derive lower bounds on optimization---existing lower bounds are for the setting where the objective $f$ coming from a very general function family, and these bounds may not apply when $f$ is restricted to be the objective of a reinforcement learning problem. In contrast, for sample complexity, we derive matching (or almost matching) upper and lower bounds. Third, the sample complexity of estimating $\nabla_{\theta}J(\theta)$ is of intrinsic interest---for example, it is an important prerequisite for safe reinforcement learning algorithms~\citep{akametalu2014reachability,berkenkamp2017safe,dean2018safely}. Finally, focusing on sample complexity simplifies our key insight. In particular, consider the the completely deterministic setting---optimizing a deterministic function using gradient descent may still take many steps, but ``estimating'' the gradient only requires a single sample.

Additionally, we note that sample complexity is directly related to the complexity of optimizing $J(\theta)$. In particular, the bounds in~\cite{vemula2019contrasting} all depend directly on the variance $\sigma^2$ of the observations $J(\theta)+\zeta$. Our proof bounds the sample complexity of estimating $\nabla J(\theta)$ by bounding the sub-Gaussian parameter of $J(\theta)$, which is an upper bound on the variance of $J(\theta)$. Thus, smaller sample complexity translates to smaller complexity of optimizing $J(\theta)$.

Finally, our focus on estimating the gradient does not address the problem of exploration. In terms of optimization, gradient estimates can be used in conjunction with gradient descent to efficiently find local minima, whereas exploration is needed to find global minima. Understanding the sample complexity of exploration is an important but orthogonal problem that we leave to future work.

\section{Preliminaries}
\label{sec:prelim}

We consider a dynamical system with states $S\subseteq\mathbb{R}^{d_S}$, actions $A\subseteq\mathbb{R}^{d_A}$, and transitions
\begin{align*}
s_{t+1}=f(s_t,a_t)+\zeta_t \hspace{0.2in} \text{where} \hspace{0.1in} \zeta_t\sim p(\zeta),
\end{align*}
where $f:S\times A\to S$ is deterministic and $\zeta\in\mathbb{R}^{d_S}$ is a random perturbation. We consider deterministic control policies $\pi_{\theta}:S\to A$ with parameters $\theta\in\Theta\subseteq\mathbb{R}^{d_{\Theta}}$. Except in the case of the model-based policy gradient algorithm, we assume that both $f$ and $p$ are unknown. We separate $f$ from $p$ since we are interested in settings where $\zeta$ is small. Also, we that assume $\zeta_t$ is independent of $s_t$ and $a_t$. This assumption enables us to substantially simplify the model-based policy gradient (since we avoid taking a derivatives of $p$), and it also simplifies our analyses of other algorithms.

We are interested in controlling the system over a finite horizon $T\in\mathbb{G}$---given a reward function $R:S\times A\to\mathbb{R}$, the goal is to find the policy $\pi_{\theta}$ that maximizes the expected cumulative reward
\begin{align*}
J(\theta)=\mathbb{E}_{p_{\theta}(\alpha)}\left[\sum_{t=0}^{T-1}R(s_t,a_t)\right],
\end{align*}
where $p_{\theta}(\alpha)$ is the distribution over rollouts $\alpha=((s_0,a_0),...,(s_{T-1},a_{T-1}))$ when using $\pi_{\theta}$, and where we assume the initial state $s_0\in S$ is deterministic and known. Note that $\alpha$ is determined by $\theta$ and $\vec{\zeta}=(\zeta_0,...,\zeta_{T-1})$, so an expectation over $p_{\theta}(\alpha)$ is equivalent to one over $p(\vec{\zeta})$. We are interested in estimating the policy gradient
\begin{align*}
D(\theta)=\nabla_{\theta}J(\theta)
\end{align*}
so we can perform gradient ascent on $\theta$. As usual, let
\begin{align*}
Q_{\theta}^{(t)}(s,a)&=\mathbb{E}_{p(\zeta)}\left[R(s,a)+V_{\theta}^{(t+1)}(f(s,a)+\zeta)\right] \\
V_{\theta}^{(t)}(s)&=Q_{\theta}^{(t)}(s,\pi_{\theta}(s)),
\end{align*}
for $t\in\{0,1,...,T-1\}$, where $V_{\theta}^{(T)}(s)=0$, denote the $Q$ function and value function, respectively~\citep{sutton2018reinforcement}. In particular, $J(\theta)=V_{\theta}^{(0)}(s_0)$.

\begin{remark}
\rm
Our results straightforwardly extend to dynamical systems with time varying dynamics and rewards. Also, we can relax our assumption that the initial state $s_0$ is deterministic---i.e., to handle an initial state distribution $p_0$, we can modify the dynamics on the first step to be $s_1=s_0+\zeta_0$, where $s_0=0$ and $\zeta_0\sim p_0$. Furthermore, our results can be extended to the case where the noise $\zeta$ appears nonlinearly in the transitions, as long as it can be reparameterized~\citep{kingma2014auto}---i.e., the transitions can be written in the form $s'=f(s,\zeta,a)$, where $\zeta\sim p(\zeta)$ i.i.d. for some $p(\zeta)$. We require that $f$ is Lipschitz in $\zeta$. Most kinds of noise considered in practice can be expressed in this form, though it may not satisfy the Lipschitz condition. Finally, our results can be extended to handle Martingale difference noise sequences by using the Azuma-Hoeffding inequality in place of the Hoeffding inequality.
\end{remark}

\section{Policy Gradient Algorithms}
\label{sec:alg}

We now describe the policy gradient estimation algorithms that we consider.

\textbf{Model-based policy gradient.}
When $f$ is known, we can estimate the policy gradient as
\begin{align*}
\nabla_{\theta}J(\theta)&=\mathbb{E}_{p(\vec{\zeta})}[\nabla_{\theta}\hat{J}(\theta;\vec{\zeta})] \\
\hat{J}(\theta;\vec{\zeta})&=\sum_{t=0}^{T-1}R(s_t,a_t).
\end{align*}
since a rollout $\alpha$ is determined by $\vec{\zeta}$. In particular, we have estimator $\nabla_{\theta}J(\theta)\approx\hat{D}_{\text{MB}}(\theta)$, where
\begin{align*}
\hat{D}_{\text{MB}}(\theta)=\frac{1}{n}\sum_{i=1}^n\hat{J}(\theta;\vec{\zeta}^{(i)})
\end{align*}
where $\vec{\zeta}^{(i)}\sim p(\vec{\zeta})$ i.i.d. for $i\in[n]$.

\textbf{Policy gradient theorem.}
The policy gradient theorem is formulated for stochastic policies---i.e., $\pi_{\theta}(a\mid s)$ is the probability of taking action $a$ in state $s$. We assume a distribution $p_{\xi}(\xi)$ of action perturbations that does not depend on $\theta$---i.e., $a_t=\pi_{\theta}(s_t)+\xi_t$, where $\xi_t\sim p_{\xi}(\xi)$. Then, we have $\tilde{\pi}_{\theta}(a\mid s)=p_{\xi}(a-\pi_{\theta}(s))$. The following are the modified $Q$ and value functions:
\begin{align*}
\tilde{Q}_{\theta}^{(t)}(s,a)&=R(s,a)+\mathbb{E}_{p(\zeta)}\left[\tilde{V}_{\theta}^{(t+1)}(f(s,a)+\zeta)\right] \\
\tilde{V}_{\theta}^{(t)}(s)&=\mathbb{E}_{\tilde{\pi}_{\theta}(a\mid s)}\left[\tilde{Q}_{\theta}^{(t)}(s,a)\right],
\end{align*}
where $\tilde{V}_{\theta}^{(T)}(s)=0$ as before. Then, the following is the policy gradient theorem~\citep{sutton2000policy}:
\begin{theorem}
\label{thm:pg}
Letting $\tilde{p}_{\theta}(\alpha)$ be the distribution over rollouts when using $\tilde{\pi}_{\theta}$, we have
\begin{align*}
\nabla_{\theta}J(\theta)=\mathbb{E}_{\tilde{p}_{\theta}(\alpha)}\left[\sum_{t=0}^{T-1}\tilde{Q}_{\theta}^{(t)}(s_t,a_t)\nabla_{\theta}\log\tilde{\pi}_{\theta}(a_t\mid s_t)\right].
\end{align*}
\end{theorem}
The key challenge to using Theorem~\ref{thm:pg} to estimate $\nabla_{\theta}J(\theta)$ is to estimate $\tilde{Q}^{(t)}(s_t,a_t)$. The simplest approach is to estimate it using a single rollout~\citep{williams1992simple}:
\begin{align*}
\tilde{Q}_{\theta}^{(t)}(s,a)&=\mathbb{E}_{\tilde{p}_{\theta}(\alpha)}\left[\hat{Q}_{\theta}^{(t)}(\alpha)\right] \\
\hat{Q}_{\theta}^{(t)}(\alpha)&=\sum_{i=t}^{T-1}R(s_i,a_i).
\end{align*}
A common technique to reduce variance is to normalize $\tilde{Q}_{\theta}^{(t)}(s,a)$ by subtracting the value function~\citep{schulman2015high}. In particular, the advantage function $\tilde{A}_{\theta}^{(t)}(s,a)=\tilde{Q}_{\theta}^{(t)}(s,a)-\tilde{V}_{\theta}^{(t)}(s)$ measures the advantage of using action $a$ instead of using $\tilde{\pi}_{\theta}$ in state $s$ at time $t$. Then, we have
\begin{align*}
\nabla_{\theta}J(\theta)=\mathbb{E}_{\tilde{p}_{\theta}(\alpha)}\left[\sum_{t=0}^{T-1}\tilde{A}_{\theta}^{(t)}(\alpha)\nabla_{\theta}\log\tilde{\pi}_{\theta}(a_t\mid s_t)\right].
\end{align*}
Unlike $\tilde{Q}_{\theta}^{(t)}$, we cannot estimate $\tilde{A}_{\theta}^{(t)}$ using a single rollout. One approach is to estimate $f_{\phi}^{(t)}(s)\approx\tilde{V}_{\theta}^{(t)}(s)$, and then estimate $\tilde{A}_{\theta}^{(t)}$ using $f_{\phi}^{(t)}$. We assume that our estimate of $\tilde{V}_{\theta}^{(t)}$ is exact---in particular, we consider the following estimator $\nabla_{\theta}J(\theta)\approx\hat{D}_{\text{PG}}(\theta)$:
\begin{align*}
\hat{D}_{\text{PG}}(\theta)&=\frac{1}{n}\sum_{i=1}^n\sum_{t=0}^{T-1}\hat{A}_{\theta}^{(t)}(\alpha^{(i)})\nabla_{\theta}\log\tilde{\pi}_{\theta}(a_t^{(i)}\mid s_t^{(i)}) \\
\hat{A}_{\theta}^{(t)}(\alpha)&=\hat{Q}_{\theta}^{(t)}(\alpha)-\tilde{V}_{\theta}^{(t)}(s_t),
\end{align*}
where $\alpha^{(i)}\sim p_{\theta}(\alpha)$ i.i.d. for each $i\in[n]$.

\begin{remark}
\rm
A common approach is to use an estimate $\hat{Q}_{\theta}^{(t)}(s,a)$ of the $Q$ function in place of $\hat{Q}_{\theta}^{(t)}(\alpha)$. This approach reduces variance, but may introduce bias. For instance, for dynamical systems with continuous actions, the deterministic policy gradient (DPG) algorithm uses this approach~\cite{silver2014deterministic}. We consider the algorithm described above for two reasons. First, our focus is on estimating the policy gradient, rather than understanding the sample complexity of $Q$-learning, which is required to analyze DPG. Second, it is hard to prove bounds for DPG since it relies on the \emph{derivative} of the $Q$ function, which cannot be bounded without additional assumptions. For example, suppose we train a random forest $\hat{Q}_{\theta}^{(t)}(s,a)$. Even if this model achieves achieves good accuracy, its gradient would be zero nearly everywhere since this model is piecewise constant; thus, it would not be useful in the context of the DPG algorithm.
\end{remark}

\textbf{Finite-difference policy gradient.}
We can use finite-differences to estimate $\nabla_{\theta}J(\theta)$.
\begin{theorem}
\label{thm:fd}
For any $f:\X\to\mathbb{R}$ (where $\X\subseteq\mathbb{R}^d$) where $\nabla f$ is $L_{\nabla f}$-Lipschitz continuous,
\footnote{We assume the $L_2$ norm throughout.}
\begin{align*}
\nabla_xf(x)&=\sum_{k=1}^d\frac{f(x+\lambda\nu^{(k)})-f(x-\lambda\nu^{(k)})}{2\lambda}\cdot\nu^{(k)}+\Delta
\end{align*}
where $\nu^{(k)}=\delta_k$ (where $\delta_k$ is the Kronecker delta), and $\Delta\in\mathbb{R}$ satisfies $\|\Delta\|\le L_{\nabla f}d\lambda$.
\end{theorem}
We give a proof in Appendix~\ref{sec:fdproof}. Then, the finite difference approximation of the policy gradient is
\begin{align*}
\nabla_{\theta}J(\theta)\approx\sum_{k=1}^{d_{\Theta}}\frac{J(\theta+\lambda\nu^{(k)})-J(\theta-\lambda\nu^{(k)})}{2\lambda}\cdot\nu^{(k)}.
\end{align*}
We can estimate $J(\theta)$ using samples $\vec{\zeta}\sim p(\vec{\zeta})$, which yields the estimator $\nabla_{\theta}J(\theta)\approx\hat{D}_{\text{FD}}(\theta)$, where
\begin{align*}
\hat{D}_{\text{FD}}(\theta)=
&\sum_{k=1}^{d_{\Theta}}\bigg[\frac{\frac{1}{n}\sum_{i=1}^n\hat{J}(\theta+\lambda\nu^{(k)};\vec{\zeta}^{(k,i)})}{2\lambda} \\
&\hspace{0.4in}-\frac{\frac{1}{n}\sum_{j=1}^n\hat{J}(\theta-\lambda\nu^{(k)};\vec{\eta}^{(k,j)})}{2\lambda}\bigg]\cdot\nu^{(k)}
\end{align*}
where $\vec{\zeta}^{(k,i)},\vec{\eta}^{(k,j)}\sim p(\vec{\zeta})$ i.i.d. for $k\in[m]$ and $i,j\in[n]$. Note that we use separate samples $\zeta^{(k,i)}$ and $\eta^{(k,j)}$ to estimate $J(\theta+\lambda\nu^{(k)})$ and $J(\theta-\lambda\nu^{(k)})$, respectively. If we are using a simulator, then we can reduce variance by using the same samples to estimate both terms.

\begin{remark}
\rm
Typically, rather than choose a fixed set of basis vectors $\nu^{(1)},...,\nu^{(k)}$, finite-difference algorithms choose random vectors from a spherically symmetric distribution---e.g., $\nu\sim\mathcal{N}(0,\sigma^2I_{d_{\theta}})$~\citep{spall1992multivariate,mania2018simple}. Our choice of a fixed basis simplifies our analysis.
\end{remark}

\section{Main Results}
\label{sec:mainresults}

\textbf{Sample complexity.}
Recall that the policy gradient $\nabla_{\theta}J(\theta)$ must be estimated from sampled rollouts $\zeta\sim p_{\theta}(\zeta)$. Our goal is to understand the tradeoffs in sample complexity of estimating $\nabla_{\theta}J(\theta)$ between various different reinforcement learning algorithms.
\begin{definition}
\rm
Let $X$ be a random vector, and let $\hat{\mu}_X^{(n)}=n^{-1}\sum_{i=1}^nx^{(i)}$, where $x^{(1)},...,x^{(n)}\sim p_X(x)$ i.i.d. The \emph{sample complexity} of $n_X(\epsilon,\delta)$ of $X$ is the smallest $n\in\mathbb{N}$ such that
\begin{align*}
\text{Pr}_{x^{(1)},...,x^{(n)}\sim p_X(x)}\left[\|\hat{\mu}_X^{(n)}\|\ge\epsilon\right]\le\delta.
\end{align*}
\end{definition}
We are interested in the sample complexity $n_{\hat{D}}$ of $\hat{D}(\zeta)-\nabla_{\theta}J(\theta)$, where $\hat{D}(\zeta)$ is an estimate of $\nabla_{\theta}J(\theta)$ using a single rollout $\zeta\sim p_{\theta}(\zeta)$.

\textbf{Assumptions.}
We let $f_{\theta}(s)=f(s,\pi_{\theta}(s))$ and $R_{\theta}(s)=R(s,\pi_{\theta}(s))$. Similarly, for a stochastic policy $\pi_{\theta}(s)+\xi$ (where $\xi\sim p(\xi)$), we let  $\tilde{f}_{\theta}(s,\xi)=f(s,\pi_{\theta}(s)+\xi)$ and $\tilde{R}_{\theta}(s)=\mathbb{E}_{p(\xi)}[R(s,\pi_{\theta}(s)+\xi)]$. Next, to ensure convergence, we make regularity assumptions about the dynamics and our control policy; see Appendix~\ref{sec:lipschitzappendix} \&~\ref{sec:subgaussianappendix} for definitions.
\begin{assumption}
\label{assump:lipschitz}
We assume that $f$, $R$, $\pi_{\theta}$, $f_{\theta}$, $\tilde{f}_{\theta}$, $R_{\theta}$ and $\tilde{R}_{\theta}$ are Lipschitz continuous and are twice continuously differentiable with Lipschitz continuous first derivative.
\end{assumption}

\begin{remark}
\rm
This standard assumption is needed to ensure that we can estimate the gradient using finite differences. It is somewhat strong---e.g., it rules out commonly used quadratic rewards. In practice, the state space is often compact, in which case the Lipschitz continuity assumption becomes redundant. However, we cannot handle discontinuous rewards or dynamics (including piecewise constant rewards). In these cases, the policy gradient may diverge near the discontinuities; thus, the sample complexity of estimating this gradient may diverge as well. In principle, we could handle discontinuities as long as the policy visits these discontinuities with zero probability.
\end{remark}


Finally, for any function $h$, we let $L_h$ denote its Lipschitz constant and $\bar{L}_h=\max\{L_{\nabla h},L_h,1\}$.
\begin{assumption}
We assume that $p(\zeta)$ is $\sigma_{\zeta}$-subgaussian.
\end{assumption}
This assumption is required for proving concentration---e.g., it is typically assumed in the context of safe reinforcement learning~\citep{akametalu2014reachability,berkenkamp2017safe}. In practice, perturbations due to noise are often bounded (which implies the noise is sub-Gaussian), especially for our setting of interest---e.g., forces due to wind, friction, or slippage have bounded magnituded. We are interested in settings where $\sigma_{\zeta}$ is small.
\begin{definition}
\rm
A system is \emph{nearly deterministic} if $\sigma_{\zeta}\ll1$.
\end{definition}
In particular, we are interested in the dependence of the sample complexity on $\sigma_{\zeta}$.

\textbf{Main theorems.}
For the model-based policy gradient, we have:
\begin{theorem}
\label{thm:mbupper}
For $\delta\le1/2$, the sample complexity of $\hat{D}_{\text{MB}}(\theta)-\nabla_{\theta}J(\theta)$ satisfies
\begin{align*}
\sqrt{n_{\text{MB}}(\epsilon,\delta)}&=\tilde{O}\left(T^8\bar{L}_{R_{\theta}}\bar{L}_{f_{\theta}}^{5T}\sigma_{\zeta}d_A/\epsilon\right) \\
\sqrt{n_{\text{MB}}(\epsilon,\delta)}&=\tilde{\Omega}\left(T\bar{L}_{f_{\theta}}^{T-3}\sigma_{\zeta}/\epsilon\right).
\end{align*}
\end{theorem}
For the policy gradient based on Theorem~\ref{thm:pg}:
\begin{theorem}
\label{thm:pgupper}
For the choice $p_{\xi}(\xi)=\mathcal{N}(\xi\mid\vec{0},\sigma_{\zeta}^2I_{d_A})$, $\hat{D}_{\text{PG}}(\theta)-\nabla_{\theta}J(\theta)$ has sample complexity
\begin{align*}
\sqrt{n_{\text{PG}}(\epsilon,\delta)}&=\tilde{O}\left(T^6(L_R+L_{\tilde{R}_{\theta}})\bar{L}_fL_{\pi}\bar{L}_{\tilde{f}_{\theta}}^Td^4/\epsilon\right),
\end{align*}
where $d=\max\{d_S,d_A\}$, for $\epsilon$ sufficiently small---i.e., $\epsilon=\Omega(T^6(L_R+L_{\tilde{R}_{\theta}})\bar{L}_fL_{\pi}\bar{L}_{\tilde{f}_{\theta}}^Td^4)$. Next,
\begin{align*}
\sqrt{n_{\text{PG}}(\epsilon,\delta)}&\ge\sqrt{n_{\xi}\left(\epsilon/\bar{L}_{f_{\theta}}^{T-2},\delta\right)} \\
\sqrt{n_{\text{PG}}(\epsilon,\delta)}&=\tilde{\Omega}\left(\min\left\{\bar{L}_{f_{\theta}}^{T/2}/\epsilon^{1/2},1/\delta^{1/2}\right\}\right).
\end{align*}
The first lower bound holds for any $p_{\xi}(\xi)$ that is everywhere differentiable on $\mathbb{R}$ and satisfies $\lim_{\xi\to\pm\infty}\xi\cdot p_{\xi}(\xi)=0$, where $n_{\xi}$ is the sample complexity of estimating $\mathbb{E}_{p_{\xi}(\xi)}[\xi\cdot\nabla_{\xi}\log p_{\xi}(\xi)]$ using samples from $p_{\xi}$. The second lower bound holds for $p_{\xi}(\xi)=\mathcal{N}(0,\sigma_{\xi}^2)$, for any $\sigma_{\xi}\in\mathbb{R}_+$.
\end{theorem}
We have shown two lower bounds---one for an arbitrary distribution $p_{\xi}$ (in terms of a sample complexity $n_{\xi}$ related to $p_{\xi}$), and one for the specific choice where $p_{\xi}$ is Gaussian (as is the case in our upper bound). Also, note that our upper bound depends on choosing the action noise to have variance $\sigma_{\zeta}$. In principle, the first lower bound holds even if $p_{\xi}$ depends on the problem parameters; however, then $n_{\xi}$ may depend on these parameters as well. The second lower bound is independent of the the action noise $\sigma_{\xi}$, so it holds even if $\sigma_{\xi}$ depends on the problem parameters.

\begin{remark}
\rm
Note that the upper and lower bounds have a gap on the order of $\epsilon^{1/2}$. We believe that this gap is due to limitations in our analysis. In particular, our lower bounds depend on a lower bound on the tail of the $\chi_n^2$ distribution, which has exponential tails. In contrast, our other lower bounds depend on Gaussian tails, which are doubly exponential. Intuitively, since the $\chi_n^2$ distribution has a longer tail, it should not have lower sample complexity.
\end{remark}

\begin{remark}
\rm
Note that the second lower bound contains a dependence on $\delta^{-1/2}$, which is unusual. However, this term only has a role if the first term in the minimum is very large. Furthermore, the first term depends as usual on $\log(1/\delta)$ (which is not shown since we omit log factors).
\end{remark}

\begin{remark}
\rm
Actor-critic approaches reduce variance by using function approximation to obtain lower variance estimates of the advantage $\tilde{A}_{\theta}^{(t)}$~\citep{schulman2015high}. However, our lower bounds hold even if the advantage is known exactly. Thus, while actor-critic approaches can reduce variance, they do not affect our main insight that these estimates remain noisy for nearly deterministic dynamical systems.
\end{remark}

For the finite-difference policy gradient:
\begin{theorem}
\label{thm:fdupper}
The sample complexity of $\hat{D}_{\text{FD}}(\theta)-\nabla_{\theta}J(\theta)$ satisfies
\begin{align*}
\sqrt{n_{\text{FD}}(\epsilon,\delta)}&=\tilde{O}\left(T^9\bar{L}_{R_{\theta}}^2\bar{L}_{f_{\theta}}^{5T}\sigma_{\zeta}d_A^2\sqrt{d_{\Theta}}/\epsilon^2\right) \\
\sqrt{n_{\text{FD}}(\epsilon,\delta)}&=\tilde{\Omega}\left(T\bar{L}_{f_{\theta}}^{3(T-3)}\sigma_{\zeta}d_{\Theta}/\epsilon^2\right).
\end{align*}
The first bound (i.e., the upper bound) holds for a choice $\lambda=O(\epsilon/T^5\bar{L}_{R_{\theta}}\bar{L}_{f_{\theta}}^{4T}d_A)$. The second bound (i.e., the lower bound) holds for any $\lambda\in\mathbb{R}_+$, $\epsilon\le1$, and $\delta\le1/2$,
\end{theorem}
Note that our upper bound is for the choice $\lambda=O(\epsilon)$, but our lower bound holds for arbitrary $\lambda$.

\begin{remark}
\rm
In an abuse of notation, in Theorem~\ref{thm:fdupper}, we have ignored the fact that $n_{\text{FD}}$ must always be at least $2d_{\Theta}$; in particular, it does not go to zero as $\sigma_{\zeta}$ goes to zero. This discrepancy in Theorem~\ref{thm:fdupper} arises because there is an implicit assumption we use when inverting Hoeffding’s inequality that $n\ge1$---more precisely, Hoeffding’s inequality gives a bound of the form
\begin{align*}
\text{Pr}[\hat{\mu}_X^{(n)}\ge\epsilon]\le\delta,
\end{align*}
where $\hat{\mu}_X^{(n)}$ is an estimate of $\mu_X=\mathbb{E}[X]$ using $n$ samples, and $\delta\ge e^{-n\epsilon^2/(2\sigma^2)}$. Solving for $n$ yields $n\ge2\sigma^2\log(1/\delta)/\epsilon^2$. However, if $\sigma=0$, then $\delta$ is not well defined, so it does not mean we can get an estimate of $\mu_X$ using $n=0$ samples; instead, we need to take $n=1$. In our proof of Theorem~\ref{thm:fdupper}, we apply Hoeffding’s inequality $2d_{\Theta}$ times (since we estimate the gradient of each component separately), so we need $n\ge2d_{\Theta}$.
\end{remark}

\textbf{Proof strategy.}
We give a high-level overview of our proof strategy, focusing on Theorem~\ref{thm:mbupper}. Our proof proceeds in two steps. First, we prove an upper bound
\begin{align}
\label{eqn:keystep}
|\hat{D}_{\text{MB}}(\theta)-\nabla_{\theta}J(\theta)|\le AE+B,
\end{align}
where $E=T^{-1}\sum_{t=0}^{T-1}\|\zeta_t\|$ and $A,B\in\mathbb{R}_+$ do not depend on $\vec{\zeta}$. This step uses induction based on the recursive structure of $V_{\theta}$. Second, we prove Lemma~\ref{lem:subgaussianbound}; we state a simplified version:
\begin{lemma}
Let $X$ be a $\sigma_X$-sub-Gaussian random vector over $\mathbb{R}^d$, and let $Y$ be a random vector over $\mathbb{R}^{d'}$ satisfying $\|Y\|\le A\|X\|_1+B$, where $A,B\in\mathbb{R}_+$. Then $Y$ is $\sigma_Y$-sub-Gaussian, where $\sigma_Y=\tilde{O}(A\sigma_Xd+B)$.
\end{lemma}
Combined with (\ref{eqn:keystep}), we conclude that $\hat{D}_{\text{MB}}(\theta)-\nabla_{\theta}J(\theta)$ is sub-Gaussian, from which we can use Hoeffding's inequality (see Lemma~\ref{lem:hoeffding}) to complete the proof. For the lower bound, we construct a system where $J(\theta)$ is Gaussian. The proof of Theorem~\ref{thm:pgupper} follows similarly, except we need to use analogous results for sub-exponential random variables. In particular, we prove Lemma~\ref{lem:subexponentialbound}, an analog of Lemma~\ref{lem:subgaussianbound}. The proof of Theorem~\ref{thm:fdupper} also follows similarly, but we need to account for the bias in the finite-difference estimate of $\nabla_{\theta}J(\theta)$ from Theorem~\ref{thm:fd}.

\begin{figure*}[t]
\centering
\begin{tabular}{ccc}
\includegraphics[width=0.3\textwidth]{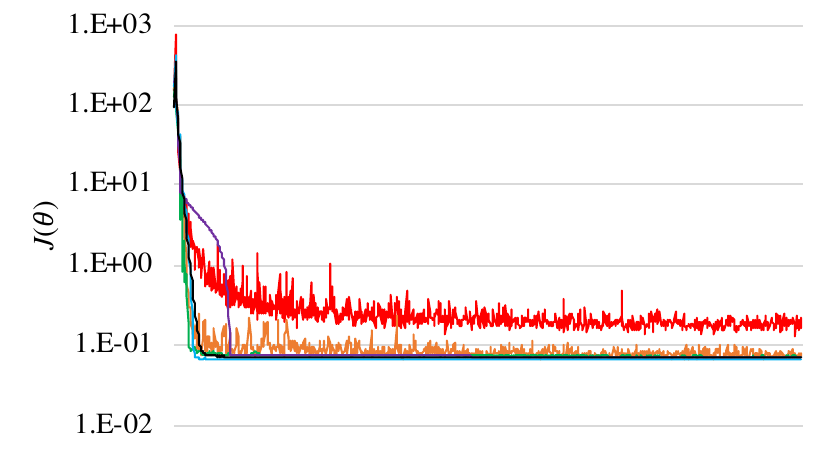} &
\includegraphics[width=0.3\textwidth]{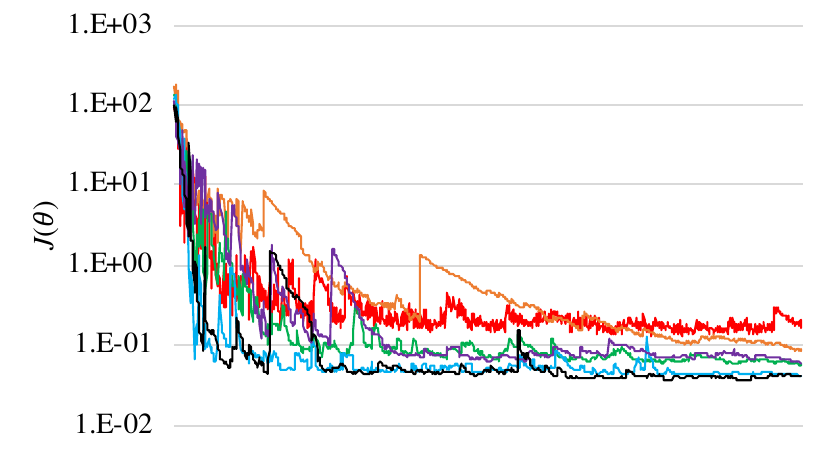} &
\includegraphics[width=0.3\textwidth]{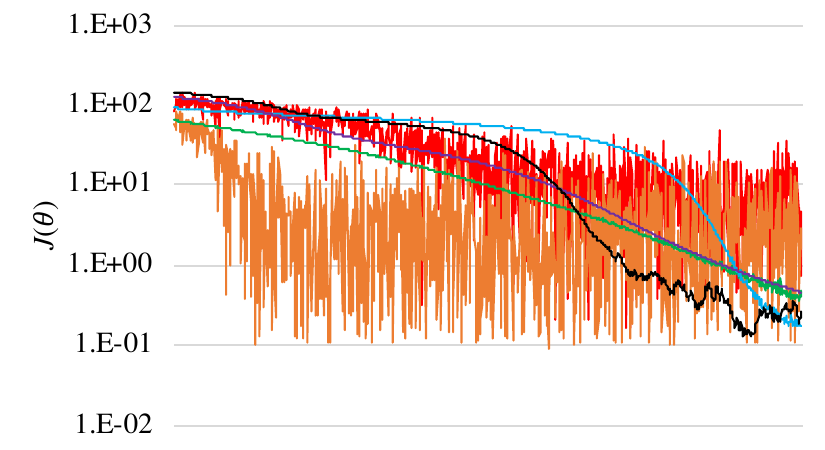} \\
Model-Based &
Finite-Difference &
Policy Gradient Theorem
\end{tabular}
\caption{The cumulative expected reward $J(\theta)$ as a function of the number of gradient steps $i\in\{1,2,...,1000\}$. The black, purple, blue, green, orange, and red curves correspond to $\sigma_{\zeta}=\{10^{-6},10^{-5},10^{-4},10^{-3},10^{-2},10^{-1}\}$, respectively. The $x$-axis is the number of gradient steps taken (or equivalently, the number of rollouts), and the $y$-axis is $-J(\theta)$.}
\label{fig:exp}
\end{figure*}

\section{Discussion}

\textbf{Dependence on $\sigma_{\zeta}$.}
Both $n_{\text{MB}}$ and $n_{\text{FD}}$ scale linearly in $\sigma_{\zeta}$. Thus, the corresponding algorithms perform very well when $\sigma_{\zeta}$ is small. In contrast, $n_{\text{PG}}$ does not become small when $\sigma_{\zeta}$ becomes small. Intuitively, if $p_{\xi}$ is wide, then the action noise adds uncertainty to $\hat{D}_{\text{PG}}(\theta)$. On the other hand, if $p_{\xi}$ is narrow, then $\nabla_{\theta}\log\tilde{\pi}_{\theta}(a\mid s)=\nabla_{\theta}\log p_{\xi}(a-\pi_{\theta}(s))$ becomes large---in particular, $p_{\xi}$ must change rapidly for some values of $\xi$, and must have large gradient at such values of $\xi$.

A key point is that in the first lower bound for $n_{\text{PG}}$ (i.e., for arbitrary $p_{\xi}$), even though we do not know its explicit dependence on $\epsilon$, $\delta$, $T$, and $\bar{L}_{f_{\theta}}$, we know that it is completely independent of $\sigma_{\zeta}$. Thus, regardless of how $p_{\xi}$ is chosen (e.g., even if it chosen based on the problem parameters), the sample complexity does not become small as $\sigma_{\zeta}$ becomes small.


\textbf{Full determinism ($\sigma_{\zeta}=0$).}
%
When $\sigma_{\zeta}=0$, we have $n_{\text{MB}}=1$ (i.e., we only need a single sample to estimate $\nabla_{\theta}J(\theta)$) and $n_{\text{FD}}=2d_{\Theta}$ (i.e., we need two samples to estimate the derivative of each parameter, taking $\lambda$ small enough to get $\epsilon$ error). For the case of $n_{\text{PG}}$, our lower bound in Theorem~\ref{thm:pgupper} still holds---the dynamical system we use to obtain the lower bound has no noise in the dynamics. In particular, a large number of samples are still needed to obtain good estimates (i.e., possibly exponential in $T$).

\textbf{Dependence on $\epsilon$.}
Both $n_{\text{MB}}$ and $n_{\text{PG}}$ depend quadratically on $\epsilon$ (ignoring the gap between the upper and lower bounds for $n_{\text{PG}}$). In contrast, $n_{\text{FD}}$ depends quartically on $\epsilon$. This gap arises because according to Theorem~\ref{thm:fd}, the finite-differences error of $\hat{D}_{\text{FD}}(\theta)$ (assuming there is no noise) depends linearly on $\lambda$. Thus, we must choose $\lambda=O(\epsilon)$ to obtain error at most $\epsilon$. If the dynamical system and control policy are both linear, then this error goes away, so the dependence on $\epsilon$ becomes quadratic.

\textbf{Dependence on $d_{\Theta}$.}
Only $n_{\text{FD}}$ depends on $d_{\Theta}$---whereas the other two algorithms make use of the fact that we can compute $\nabla_{\theta}\pi_{\theta}$, the finite-difference approximation ignores this ability.

\textbf{Dependence on $T$.}
All of the sample complexities depend exponentially on $T$. As we show in our lower bounds, this dependence is unavoidable---it arises from the fact that the dynamics cause the state (and therefore the rewards) to grow exponentially large in $T$. A common assumption made in prior work is that the rewards are bounded uniformly by $R_{\text{max}}\in\mathbb{R}_+$~\citep{kearns2002near,kakade2003sample}. Intuitively, our results indicate that without stronger assumptions, $R_{\text{max}}$ may be exponentially large. In practice, rewards for continuous control tasks are often quadratic, and can indeed be exponentially in magnitude.

An important aspect is that estimation is substantially easier when the current policy is good. In our bounds, the base of the exponential dependence is always $\bar{L}_{f_{\theta}}$. If the initial policy $\pi_{\theta}$ provides relatively stable control, then we may expect that $L_{f_{\theta}}\le1$---i.e., the states remain bounded in magnitude. Then, we have $\bar{L}_{f_{\theta}}=1$, so our bounds no longer depend exponentially on $T$. This insight suggests the importance of good initialization for fast estimation.

Indeed, policy gradient estimators can have high variance in practice. As an example, consider the cart-pole problem with continuous action space, with random initial state and where the reward function is the negative distance to origin. We empirically estimated that the MSE of the model-based policy gradient estimator using $n=1$ on a randomly initialized policy for this benchmark is $3.5\times10^7$. This error is substantially reduced when the policy is stable---for a trained cart-pole policy, we estimate that the MSE of the model-based policy gradient estimator is just $5.2\times10^{-2}$.

\section{Experiments}

We empirically evaluated the effect of $\sigma_{\zeta}$ on the performance of the different algorithms.

\textbf{Dynamical system.}
We use the inverted pendulum ~\citep{tedrake2018underactuated} (specifically, using the dynamics from OpenAI Gym~\citep{brockman2016openai}), which has state space $S=\mathbb{R}^2$ (i.e., angle $\vartheta$ and angular velocity $\omega$) and actions $A=\mathbb{R}$ (i.e., applied torque). Letting $f$ be the (deterministic) pendulum dynamics, we consider the system $s_{t+1}=f(s_t,a_t)+\zeta_t$, where $\zeta_t\sim\mathcal{N}(0,\sigma_{\zeta}^2)$ i.i.d. We use the rewards
\begin{align*}
R((\vartheta,\omega),a)=-(w_{\vartheta}\cdot(\vartheta-\vartheta_0)^2+w_{\omega}\cdot\omega^2+w_a\cdot a^2),
\end{align*}
where $\vartheta_0$ is the angle corresponding to the upright position, and $w_{\vartheta}=1$, $w_{\omega}=10^{-1}$, and $w_a=10^{-2}$. Our goal is to control the system over a horizon of $T=50$ steps, from a fixed start state $s_0=(\vartheta_0',0)$, where $\vartheta_0'=0.05$. For the control policy, we used a neural network $\pi_{\theta}$ with a single hidden layer with 100 neurons, ReLU activations, and linear outputs. As usual, we randomly initialize the weights; to reduce variance, we initialized the policy to have a reasonably high reward by running our model-based algorithm until $J(\theta)\ge-100$.

\textbf{Algorithms.}
We use stochastic gradient descent in conjunction with each of the three estimation algorithms. On each gradient step, we use a single sample to estimate the gradient, and we take 1000 gradient steps. We modify the finite-difference algorithm to use a single random sample $\nu\sim\text{Uniform}(S^{d_{\Theta}-1})$ (i.e., the uniform distribution on the unit sphere in $\mathbb{R}^{d_{\Theta}}$), rather than summing over the $d_{\Theta}$ basis vectors $\nu^{(k)}$. This choice may improve the dependence of the sample complexity on $d_{\Theta}$; however, it should not affect dependence on $\sigma_{\zeta}$, which is our parameter of interest.

For the algorithm based on the policy gradient theorem, we use action noise $\xi\sim\mathcal{N}(0,\sigma_{\xi}I_{d_A})$. For each choice of $\sigma_{\zeta}$, we used cross-validation to identify the optimal hyperparameters: the learning rate $\upsilon$ (for all algorithms), the parameter $\lambda$ (for the finite-differences algorithm), and the action noise $\sigma_{\xi}$ (for the algorithm based on the policy gradient theorem).

%

\textbf{Results.}
Average the results of each algorithm over 20 runs; the algorithms have very high variance, so we discard runs that do not converge. In Figure~\ref{fig:exp}, we show the learning curves for $\sigma_{\zeta}\in\{10^{-6},10^{-5},10^{-4},10^{-3},10^{-2},10^{-1}\}$ (i.e., $J(\theta)$ as a function of the number of gradient steps). The darker colors correspond to smaller noise. We show enlarged versions of these plots in Appendix~\ref{sec:expappendix}.

Note that unlike the other two algorithms, the finite-difference algorithm actually uses 2000 sampled rollouts (since it uses two per gradient step). However, this detail does not affect our insights regarding the relative convergence rate of different algorithms for different $\sigma_{\zeta}$.

Our key finding is that the learning curves for the model-based and finite-differences are ordered based on the choice of $\sigma_{\zeta}$---i.e., the curves tend to converge more quickly for smaller choices of $\sigma_{\zeta}$. This effect is most apparent in the curves for the finite-differences algorithms, where curves for smaller $\sigma_{\zeta}$ (black and blue) converge much faster than those for larger $\sigma_{\zeta}$ (red and orange). In contrast, the learning curves for the policy gradient based algorithm do not have strong dependence on $\sigma_{\zeta}$. For example, the fastest curve to converge (at least initially) for the policy gradient based algorithm is for our second-largest choice $\sigma_{\zeta}=10^{-2}$ (orange), whereas the slowest to converge is for $\sigma_{\zeta}=10^{-4}$ (blue). These results mirror our theoretical insights.

Finally, as expected, the model-based algorithm converges most quickly, followed by the finite-differences and policy gradient theorem based algorithms.

\section{Conclusion}

We have analyzed the sample complexity of algorithms for estimating the policy gradient for nearly deterministic dynamical systems. Future work includes leveraging these results in safe reinforcement learning algorithms, and understanding the sample complexity of optimizing $J(\theta)$.

\subsubsection*{Acknowledgements}

This work was supported by NSF CCF-1910769.

\bibliography{paper}
\bibliographystyle{plainnat}

\clearpage
\onecolumn
\appendix
\section{Proof of Theorem~\ref{thm:mbupper}}
\label{sec:mbupperproof}

\paragraph{Preliminaries.}

Note that the expected cumulative reward is equivalent to
\begin{align*}
J(\theta)&=V_{\theta}^{(0)}(s_0) \\
V_{\theta}^{(t)}(s)&=R_{\theta}(s)+\mathbb{E}_{p(\zeta)}\left[V_{\theta}^{(t+1)}(f_{\theta}(s)+\zeta)\right]\hspace{0.2in}(\forall t\in\{0,1,...,T-1\}) \\
V_{\theta}^{(T)}(s)&=0
\end{align*}
and the expected model-based policy gradient is
\begin{align*}
\nabla_{\theta}J(\theta)&=\nabla_{\theta}V_{\theta}^{(0)}(s_0) \\
\nabla_{\theta}V_{\theta}^{(t)}(s)
&=\nabla_{\theta}R_{\theta}(s)+\mathbb{E}_{p(\zeta)}\left[\nabla_{\theta}V_{\theta}^{(t+1)}(f_{\theta}(s)+\zeta)+\nabla_sV_{\theta}^{(t+1)}(f_{\theta}(s)+\zeta)\nabla_{\theta}f_{\theta}(s)\right] \\
\nabla_sV_{\theta}^{(t)}(s)
&=\nabla_sR_{\theta}(s)+\mathbb{E}_{p(\zeta)}\left[\nabla_sV_{\theta}^{(t+1)}(f_{\theta}(s)+\zeta)\nabla_sf_{\theta}(s)\right] \\
\nabla_{\theta}V_{\theta}^{(T)}(s)&=\nabla_sV_{\theta}^{(T)}(s)=0.
\end{align*}
Similarly, given a sample $\vec{\zeta}\sim p(\vec{\zeta})$, the stochastic approximation of the expected cumulative reward is
\begin{align*}
\hat{J}(\theta;\vec{\zeta})&=\hat{V}_{\theta}^{(0)}(s_0;\vec{\zeta}) \\
\hat{V}_{\theta}^{(t)}(s;\vec{\zeta})&=R_{\theta}(s)+\hat{V}_{\theta}^{(t+1)}(f_{\theta}(s)+\zeta_t;\vec{\zeta})\hspace{0.2in}(\forall t\in\{0,1,...,T-1\}) \\
\hat{V}_{\theta}^{(T)}(s;\vec{\zeta})&=0
\end{align*}
and the stochastic approximation of the model-based policy gradient is
\begin{align*}
\nabla_{\theta}\hat{J}(\theta;\vec{\zeta})&=\nabla_{\theta}\hat{V}_{\theta}^{(0)}(s_0;\vec{\zeta}) \\
\nabla_{\theta}\hat{V}_{\theta}^{(t)}(s;\vec{\zeta})&=\nabla_{\theta}R_{\theta}(s)+\nabla_{\theta}\hat{V}_{\theta}^{(t+1)}(f_{\theta}(s)+\zeta_t;\vec{\zeta})+\nabla_s\hat{V}_{\theta}^{(t+1)}(f_{\theta}(s)+\zeta_t;\vec{\zeta})\nabla_{\theta}f_{\theta}(s) \\
\nabla_s\hat{V}_{\theta}^{(t)}(s;\vec{\zeta})&=\nabla_sR_{\theta}(s)+\nabla_s\hat{V}_{\theta}^{(t+1)}(f_{\theta}(s)+\zeta_t;\vec{\zeta})\nabla_sf_{\theta}(s) \\
\nabla_{\theta}\hat{V}_{\theta}^{(T)}(s;\vec{\zeta})&=\nabla_s\hat{V}_s^{(T)}(s;\vec{\zeta})=0.
\end{align*}

\paragraph{Bounding the deviation of $\nabla_{\theta}\hat{V}_{\theta}^{(t)}$ from $\nabla{\theta}V_{\theta}^{(t)}$.}

We claim that for $t\in\{0,1,...,T\}$, we have
\begin{align*}
\|\nabla_{\theta}\hat{V}_{\theta}^{(t)}(s;\vec{\zeta})-\nabla_{\theta}V_{\theta}^{(t)}(s)\|&\le B_0^{(t)}(\vec{\zeta}) \\
\|\nabla_s\hat{V}_{\theta}^{(t)}(s;\vec{\zeta})-\nabla_sV_{\theta}^{(t)}(s)\|&\le B_1^{(t)}(\vec{\zeta})
\end{align*}
for all $\theta\in\Theta$ and $s\in S$, where
\begin{align*}
B_0^{(t)}(\vec{\zeta})&=\sum_{i=t}^{T-1}L_{f_{\theta}}B_1^{(i+1)}(\vec{\zeta})+L_{\nabla V}^{(i+1)}(L_{f_{\theta}}+1)(\|\zeta_i\|+\sigma_{\zeta}\sqrt{d_S}) \\
B_1^{(t)}(\vec{\zeta})&=\sum_{i=t}^{T-1}L_{\nabla V}^{(i+1)}L_{f_{\theta}}^{i-t+1}(\|\zeta_i\|+\sigma_{\zeta}\sqrt{d_S}) \\
B_0^{(T)}(\vec{\zeta})&=B_1^{(T)}(\vec{\zeta})=0,
\end{align*}
where $L_{\nabla V}^{(t)}$ is a Lipschitz constant for $\nabla V_{\theta}^{(t)}$. The base case $t=T$ follows trivially. Note that $\sigma_{\zeta}\sqrt{d_S}\ge\sqrt{\mathbb{E}_{p(\zeta)}[\|\zeta\|^2]}\ge\mathbb{E}_{p(\zeta)}[\|\zeta\|]$. Then, for $t\in\{0,1,...,T-1\}$, we have
\begin{align*}
\|\nabla_{\theta}\hat{V}_{\theta}^{(t)}(s;\vec{\zeta})-\nabla_{\theta}V_{\theta}^{(t)}(s)\|
\le&\left\|\nabla_{\theta}\hat{V}_{\theta}^{(t+1)}(f_{\theta}(s)+\zeta_t;\vec{\zeta})-\mathbb{E}_{p(\zeta)}\left[\nabla_{\theta}V_{\theta}^{(t+1)}(f_{\theta}(s)+\zeta)\right]\right\| \\
&+L_{f_{\theta}}\left\|\nabla_s\hat{V}_{\theta}^{(t+1)}(f_{\theta}(s)+\zeta_t;\vec{\zeta})-\mathbb{E}_{p(\zeta)}\left[\nabla_sV_{\theta}^{(t+1)}(f_{\theta}(s)+\zeta)\right]\right\| \\
\le&\left\|\nabla_{\theta}\hat{V}_{\theta}^{(t+1)}(f_{\theta}(s)+\zeta_t;\vec{\zeta})-\nabla_{\theta}V_{\theta}^{(t+1)}(f_{\theta}(s)+\zeta_t)\right\| \\
&+\mathbb{E}_{p(\zeta)}\left[\left\|\nabla_{\theta}V_{\theta}^{(t+1)}(f_{\theta}(s)+\zeta_t)-\nabla_{\theta}V_{\theta}^{(t+1)}(f_{\theta}(s)+\zeta)\right\|\right] \\
&+L_{f_{\theta}}\left\|\nabla_s\hat{V}_{\theta}^{(t+1)}(f_{\theta}(s)+\zeta_t;\vec{\zeta})-\nabla_sV_{\theta}^{(t+1)}(f_{\theta}(s)+\zeta_t)\right\| \\
&+L_{f_{\theta}}\mathbb{E}_{p(\zeta)}\left[\left\|\nabla_sV_{\theta}^{(t+1)}(f_{\theta}(s)+\zeta_t)-\nabla_sV_{\theta}^{(t+1)}(f_{\theta}(s)+\zeta)\right\|\right] \\
\le&B_0^{(t+1)}(\vec{\zeta})+L_{\nabla V}^{(t+1)}(\|\zeta_t\|+\sigma_{\zeta}\sqrt{d_S})+L_{f_{\theta}}B_1^{(t+1)}(\vec{\zeta})+L_{f_{\theta}}L_{\nabla V}^{(t+1)}(\|\zeta_t\|+\sigma_{\zeta}\sqrt{d_S}) \\
=&B_0^{(t+1)}(\vec{\zeta})+L_{f_{\theta}}B_1^{(t+1)}(\vec{\zeta})+L_{\nabla V}^{(t+1)}(L_{f_{\theta}}+1)(\|\zeta_t\|+\sigma_{\zeta}\sqrt{d_S}) \\
=&B_0^{(t)}(\vec{\zeta}).
\end{align*}
Similarly, we have
\begin{align*}
\|\nabla_s\hat{V}_{\theta}^{(t)}(s;\vec{\zeta})-\nabla_sV_{\theta}^{(t)}(s)\|
\le&L_{f_{\theta}}\left\|\nabla_s\hat{V}_{\theta}^{(t+1)}(f_{\theta}(s)+\zeta_t;\vec{\zeta})-\mathbb{E}_{p(\zeta)}\left[\nabla_sV_{\theta}^{(t+1)}(f_{\theta}(s)+\zeta)\right]\right\| \\
\le&L_{f_{\theta}}\left\|\nabla_s\hat{V}_{\theta}^{(t+1)}(f_{\theta}(s)+\zeta_t;\vec{\zeta})-\nabla_sV_{\theta}^{(t+1)}(f_{\theta}(s)+\zeta_t)\right\| \\
&+L_{f_{\theta}}\mathbb{E}_{p(\zeta)}\left[\left\|\nabla_sV_{\theta}^{(t+1)}(f_{\theta}(s)+\zeta_t)-\nabla_sV_{\theta}^{(t+1)}(f_{\theta}(s)+\zeta)\right\|\right] \\
\le&L_{f_{\theta}}\left(B_1^{(t+1)}(\vec{\zeta})+L_{\nabla V}^{(t+1)}(\|\zeta_t\|+\sigma_{\zeta}\sqrt{d_S})\right) \\
=&B_1^{(t)}(\vec{\zeta}).
\end{align*}
The claim follows.

\paragraph{Bounding the deviation of $\nabla_{\theta}\hat{J}$ from $\nabla_{\theta}J$.}

We claim that
\begin{align*}
\|\nabla_{\theta}\hat{J}(\theta;\vec{\zeta})-\nabla_{\theta}J(\theta)\|
\le132T^7\bar{L}_{R_{\theta}}\bar{L}_{f_{\theta}}^{5T}(E+\sigma_{\zeta}\sqrt{d_S}),
\end{align*}
where $E=T^{-1}\sum_{t=0}^{T-1}\|\zeta_t\|$. To this end, letting $L_{\nabla V}=\operatorname*{\arg\max}_{t\in\{0,1,...,T\}}L_{\nabla V}^{(t)}$, note that
\begin{align*}
B_1^{(t)}\le TL_{\nabla V}\bar{L}_{f_{\theta}}^{T-1}(E+\sigma_{\zeta}\sqrt{d_S})
\end{align*}
for $t\in\{1,2,...,T\}$, so
\begin{align*}
\|\nabla_{\theta}\hat{J}(\theta;\vec{\zeta})-\nabla_{\theta}J(\theta)\|
\le B_0^{(0)}(\vec{\zeta})
&=\sum_{i=0}^{T-1}L_{f_{\theta}}B_1^{(i+1)}(\vec{\zeta})+L_{\nabla V}(L_{f_{\theta}}+1)(\|\zeta_i\|+\sigma_{\zeta}\sqrt{d_S}) \\
&\le T^2L_{\nabla V}\bar{L}_{f_{\theta}}^T(E+\sigma_{\zeta}\sqrt{d_S})+TL_{\nabla V}(L_{f_{\theta}}+1)(E+\sigma_{\zeta}\sqrt{d_S}) \\
&\le3T^2L_{\nabla V}\bar{L}_{f_{\theta}}^T(E+\sigma_{\zeta}\sqrt{d_S}) \\
&\le132T^7\bar{L}_{R_{\theta}}\bar{L}_{f_{\theta}}^{5T}(E+\sigma_{\zeta}\sqrt{d_S}),
\end{align*}
where the last step follows from our bound on $L_{\nabla V}^{(t)}$ in Lemma~\ref{lem:nablavlipschitz}.

\paragraph{Upper bound on sample complexity of $\nabla_{\theta}\hat{J}-\nabla_{\theta}J$.}

Note that $E\le\|\vec{\zeta}\|_1$, where we think of $\vec{\zeta}$ as the length $Td_S$ concatenation of the vectors $\zeta_0,\zeta_1,...,\zeta_{T-1}$, so $\vec{\zeta}$ is $\sigma_{\zeta}$-sub-Gaussian. We apply Lemma~\ref{lem:subgaussianbound} with
\begin{align*}
Y&=\nabla_{\theta}\hat{J}(\theta;\vec{\zeta})-\nabla_{\theta}J(\theta) \\
X&=E \\
A&=132T^7\bar{L}_{R_{\theta}}\bar{L}_{f_{\theta}}^{5T} \\
B&=A\sigma_{\zeta}\sqrt{d_S}.
\end{align*}
Thus, $Y$ is $\sigma_{\text{MB}}$-sub-Gaussian, where
\begin{align*}
\sigma_{\text{MB}}
&=\max\{10A\sigma_{\zeta}Td_S\log(Td_S),5A\sigma_{\zeta}\sqrt{d_S}\} \\
&=10A\sigma_{\zeta}Td_S\log(Td_S) \\
&\le1320T^8\bar{L}_{R_{\theta}}\bar{L}_{f_{\theta}}^{5T}\sigma_{\zeta}d_S\log(Td_S).
\end{align*}
Thus, by Lemma~\ref{lem:hoeffdingvector}, the sample complexity of $\nabla_{\theta}\hat{J}(\theta)-\nabla_{\theta}J(\theta)$ is
\begin{align*}
\sqrt{n_{\text{MB}}(\epsilon,\delta)}&=\frac{\sigma_{\text{MB}}\sqrt{2\log(2d_S/\delta)}}{\epsilon} \\
&=O\left(\frac{T^8\bar{L}_{R_{\theta}}\bar{L}_{f_{\theta}}^{5T}\sigma_{\zeta}d_S\log(T)\log(d_S)^{3/2}\log(1/\delta)^{1/2}}{\epsilon}\right).
\end{align*}
The claim follows.

\paragraph{Lower bound on sample complexity of $\nabla_{\theta}\hat{J}-\nabla_{\theta}J$.}

Consider a linear dynamical system with $S=A=\mathbb{R}$, time-invariant deterministic transitions $f(s,a)=\beta s+a$ (where $\beta\in\mathbb{R}$), time-varying noise
\begin{align*}
p_t(\zeta)=\begin{cases}\mathcal{N}(\zeta\mid 0,\sigma_{\zeta}^2)&\text{if}~t=0\\\delta(0)&\text{otherwise},\end{cases}
\end{align*}
where $\sigma_{\zeta}\in\mathbb{R}$, initial state $s_0=0$, time-varying rewards
\begin{align*}
R_t(s,a)=\begin{cases}s&\text{if}~t=T-1\\0&\text{otherwise},\end{cases}
\end{align*}
control policy class $\pi_{\theta}(s)=\theta s$, and current parameters $\theta=0$. Note that
\begin{align*}
s_t=\begin{cases}0&\text{if}~t=0\\(\beta+\theta)^{t-1}\zeta&\text{otherwise},\end{cases}
\end{align*}
where $\zeta=\zeta_0$ is the noise on the first step. Thus, we have
\begin{align*}
\hat{J}(\theta;\zeta)=s_{T-1}=(\beta+\theta)^{T-2}\zeta,
\end{align*}
so
\begin{align*}
\nabla_{\theta}\hat{J}(\theta;\zeta)=(T-2)(\beta+\theta)^{T-3}\zeta.
\end{align*}
Also, note that
\begin{align*}
\nabla_{\theta}J(\theta)=\mathbb{E}_{p(\zeta)}[\nabla_{\theta}\hat{J}(0;\zeta)]=\mathbb{E}_{p(\zeta)}[(T-2)(\beta+\theta)^{T-3}\zeta]=0.
\end{align*}
Next, note that for $n$ i.i.d. samples $\zeta^{(1)},...,\zeta^{(n)}\sim\mathcal{N}(0,\sigma_{\zeta}^2)$, we have
\begin{align*}
\hat{D}_{\text{MB}}(0)-\nabla_{\theta}J(0)&=\frac{1}{n}\sum_{i=1}^n(T-2)\beta^{T-3}\zeta^{(i)}\sim\mathcal{N}\left(0,\frac{\sigma_{\text{MB}}^2}{n}\right),
\end{align*}
where
\begin{align*}
\sigma_{\text{MB}}=\sigma_{\zeta}^2(T-2)^2\beta^{2(T-3)}.
\end{align*}
Thus, by Lemma~\ref{lem:gaussianlower}, for
\begin{align*}
n<\frac{\sigma_{\text{MB}}^2\left(\log\left(\sqrt{\frac{e}{2\pi}}\right)+\log(1/\delta)\right)}{\epsilon^2},
\end{align*}
we have
\begin{align*}
\text{Pr}\left[|\hat{D}_{\text{MB}}(0)-\nabla_{\theta}J(0)|\ge\epsilon\right]=\text{Pr}_{x\sim\mathcal{N}(0,\sigma_{\text{MB}}^2/n)}\left[|x|\ge\epsilon\right]\ge\sqrt{\frac{e}{2\pi}}\cdot e^{-n\epsilon^2/\sigma_{\text{MB}}^2}>\delta.
\end{align*}
Thus, the sample complexity of $\hat{D}_{\text{MB}}(0)-\nabla_{\theta}J(0)$ satisfies
\begin{align*}
n_{\text{MB}}(\epsilon,\delta)\ge\frac{\sigma_{\zeta}^2(T-2)^2\beta^{2(T-3)}\cdot\left(\log\left(\sqrt{\frac{e}{2\pi}}\right)+\log(1/\delta)\right)}{\epsilon^2}.
\end{align*}
Note that the numerator is positive as long as $\delta\le1/2$. The claim follows, as does the theorem statement. $\qed$

\section{Proof of Theorem~\ref{thm:pgupper}}
\label{sec:pgupperproof}

\paragraph{Preliminaries.}

Recall the form of the policy gradient based on Theorem~\ref{thm:pg}:
\begin{align*}
\nabla_{\theta}J(\theta)&=\mathbb{E}_{\tilde{p}_{\theta}(\zeta)}\left[\sum_{t=0}^{T-1}\hat{A}_{\theta}^{(t)}(\zeta)\nabla_{\theta}\log\tilde{\pi}_{\theta}(a_t\mid s_t)\right],
\end{align*}
where, for $t\in\{0,1,...,T-1\}$, we have
\begin{align*}
\hat{A}_{\theta}^{(t)}(\alpha)&=\hat{Q}_{\theta}^{(t)}(\alpha)-\tilde{V}_{\theta}^{(t)}(s_t),
\end{align*}
where
\begin{align*}
\hat{Q}_{\theta}^{(t)}(\alpha)&=R(s_t,a_t)+\hat{Q}_{\theta}^{(t+1)}(\alpha) \\
\tilde{V}_{\theta}^{(t)}(s)&=\mathbb{E}_{p_{\xi}(\xi),p(\zeta)}[\tilde{R}_{\theta}(s)+\tilde{V}_{\theta}^{(t+1)}(\tilde{f}_{\theta}(s,\xi)+\zeta)] \\
\hat{Q}_{\theta}^{(T)}(\alpha)&=\tilde{V}_{\theta}^{(T)}(s)=0.
\end{align*}
The stochastic approximation of $\nabla_{\theta}J(\theta)$ for a single sampled rollout $\alpha\sim\tilde{p}(\alpha)$ is
\begin{align*}
\hat{D}_{\text{PG}}(\theta;\alpha)&=\sum_{t=0}^{T-1}\hat{A}_{\theta}^{(t)}(\alpha)\nabla_{\theta}\log\tilde{\pi}_{\theta}(a_t\mid s_t).
\end{align*}

\paragraph{Bounding $\hat{Q}_{\theta}^{(t)}-\tilde{V}_{\theta}^{(t)}$.}

We claim that
\begin{align*}
\|\hat{Q}_{\theta}^{(t)}(\zeta)-\tilde{V}_{\theta}^{(t)}(s_t)\|\le&B^{(t)}(\zeta),
\end{align*}
where
\begin{align*}
B^{(t)}(\zeta)&=\sum_{i=t}^{T-1}(L_R+L_{\tilde{V}}^{(i+1)}L_f)(\|\xi_t\|+\sigma_{\zeta}\sqrt{d})+L_{\tilde{V}}^{(i+1)}(\|\zeta_t\|+\sigma_{\zeta}\sqrt{d}),
\end{align*}
where $L_{\tilde{V}}^{(t)}$ is a Lipschitz constant for $\tilde{V}_{\theta}^{(t)}$. We prove by induction. The base case $t=T$ is trivial. Note that $\sigma_{\zeta}\sqrt{d}\ge\sqrt{\mathbb{E}_{p(\zeta)}[\|\zeta\|^2]}\ge\mathbb{E}_{p(\zeta)}[\|\zeta\|]$, and similarly $\sigma_{\zeta}\sqrt{d}\ge\sqrt{\mathbb{E}_{p_{\xi}(\xi)}[\|\xi\|^2]}\ge\mathbb{E}_{p_{\xi}(\xi)}[\|\xi\|]$. Then, for $t\in\{0,1,...,T-1\}$, we have
\begin{align*}
\|\hat{Q}_{\theta}^{(t)}(\zeta)-\tilde{V}_{\theta}^{(t)}(s_t)\|
\le&\mathbb{E}_{p_{\xi}(\xi)}\left[\|R(s_t,\pi_{\theta}(s_t)+\xi_t)-R(s_t,\pi_{\theta}(s_t)+\xi)\|\right] \\
&+\|\hat{Q}_{\theta}^{(t+1)}(\zeta)-\tilde{V}_{\theta}^{(t+1)}(s_{t+1})\| \\
&+\mathbb{E}_{p_{\xi}(\xi),p(\zeta)}\left[\|\tilde{V}_{\theta}^{(t+1)}(f(s_t,\pi_{\theta}(s_t)+\xi_t)+\zeta_t)-\tilde{V}_{\theta}^{(t+1)}(f(s_t,\pi_{\theta}(s_t)+\xi)+\zeta)\|\right] \\
\le&L_R(\|\xi_t\|+\sigma_{\zeta}\sqrt{d})+B^{(t+1)}(\zeta)+L_{\tilde{V}}^{(t+1)}(\|\zeta_t\|+\sigma_{\zeta}\sqrt{d})+L_{\tilde{V}}^{(t+1)}L_f(\|\xi_t\|+\sigma_{\zeta}\sqrt{d}) \\
=&B^{(t)}(\zeta).
\end{align*}
The claim follows.

\paragraph{Bounding $\log\tilde{\pi}_{\theta}(a\mid s)$.}

We claim that
\begin{align*}
\|\nabla_{\theta}\log\tilde{\pi}_{\theta}(a\mid s)\|\le\frac{L_{\pi}}{\sigma_{\zeta}^2}\cdot\|\xi\|,
\end{align*}
where $\xi=a-\pi_{\theta}(s)$. Recall that $p_{\xi}(\xi)=\mathcal{N}(\vec{0},\sigma_{\zeta}^2I_{d_A})$. Thus, we have
\begin{align*}
\log\tilde{\pi}_{\theta}(a\mid s)=\log p_{\xi}(a-\pi_{\theta}(s))&=\log\mathcal{N}(a-\pi_{\theta}(s)\mid0,\sigma_{\zeta}^2I_{d_A})=-\frac{1}{2}\log(2\pi\sigma_{\zeta}^2)-\frac{1}{2\sigma_{\zeta}^2}\cdot\|a-\pi_{\theta}(s)\|^2.
\end{align*}
Thus, we have
\begin{align*}
\|\nabla_{\theta}\log\tilde{\pi}_{\theta}(a\mid s)\|=\frac{1}{2\sigma_{\zeta}^2}\cdot\left\|\nabla_{\theta}\|a-\pi_{\theta}(s)\|^2\right\|=\frac{1}{\sigma_{\zeta}^2}\cdot\left\|\nabla_{\theta}\pi_{\theta}(s)^{\top}(a-\pi_{\theta}(s))\right\|\le\frac{L_{\pi}}{\sigma_{\zeta}^2}\cdot\|\xi\|,
\end{align*}
as claimed.

\paragraph{Bounding the deviation of $\hat{D}_{\text{PG}}$ from $\nabla_{\theta}J$.}

We claim that
\begin{align*}
\|\hat{D}_{\text{PG}}(\theta;\zeta)-\nabla_{\theta}J(\theta)\|\le3T^4(L_R+L_{\tilde{R}_{\theta}})\bar{L}_fL_{\pi}\bar{L}_{\tilde{f}_{\theta}}^Td\cdot\left(4d+\frac{\tilde{E}+E+2\sigma_{\zeta}\sqrt{d}}{\sigma_{\zeta}^2}\right),
\end{align*}
where $L_{\tilde{V}}=\operatorname{\arg\max}_{t\in\{1,...,T\}}L_{\tilde{V}}^{(t)}$, $E=T^{-1}\sum_{t=0}^{T-1}\|\zeta_t\|$, and $\tilde{E}=T^{-1}\sum_{t=0}^{T-1}\|\xi_t\|$. First, note that
\begin{align*}
\|\hat{Q}_{\theta}^{(t)}(\zeta)-\tilde{V}_{\theta}^{(t)}(s_t)\|&\le T\left((L_R+L_{\tilde{V}}L_f)(\tilde{E}+\sigma_{\zeta}\sqrt{d})+L_{\tilde{V}}(E+\sigma_{\zeta}\sqrt{d})\right) \\
&\le3T^3(L_R+L_{\tilde{R}_{\theta}})\bar{L}_f\bar{L}_{\tilde{f}_{\theta}}^{T-1}(\tilde{E}+E+2\sigma_{\zeta}\sqrt{d}),
\end{align*}
where the last step follows from the bound on $L_{\tilde{V}}^{(t)}$ in Lemma~\ref{lem:tildevlipschitz}. Then, we have
\begin{align*}
\|\hat{D}_{\text{PG}}(\theta;\zeta)\|&=\left\|\sum_{t=0}^{T-1}(\hat{Q}_{\theta}^{(t)}(\zeta)-\tilde{V}_{\theta}^{(t)}(s_t))\nabla_{\theta}\log\tilde{\pi}_{\theta}(a_t\mid s_t)\right\| \\
&\le\sum_{t=0}^{T-1}\|\hat{Q}_{\theta}^{(t)}(\zeta)-\tilde{V}_{\theta}^{(t)}(s_t)\|\cdot\|\nabla_{\theta}\log\tilde{\pi}_{\theta}(a_t\mid s_t)\| \\
&\le\sum_{t=0}^{T-1}3T^3(L_R+L_{\tilde{R}_{\theta}})\bar{L}_f\bar{L}_{\tilde{f}_{\theta}}^T(\tilde{E}+E+2\sigma_{\zeta}\sqrt{d})\cdot\frac{L_{\pi}}{\sigma_{\zeta}^2}\cdot\|\xi_t\| \\
&=3T^4(L_R+L_{\tilde{R}_{\theta}})\bar{L}_fL_{\pi}\bar{L}_{\tilde{f}_{\theta}}^T\cdot\frac{(E+\tilde{E}+2\sigma_{\zeta}\sqrt{d})\tilde{E}}{\sigma_{\zeta}^2}.
\end{align*}
Furthermore, we have
\begin{align*}
\|\nabla_{\theta}J(\theta)\|&\le\mathbb{E}_{\tilde{p}_{\theta}(\zeta)}[\|\hat{D}_{\text{PG}}(\theta;\zeta)\|] \\
&\le\mathbb{E}_{\tilde{p}_{\theta}(\zeta)}\left[3T^4(L_R+L_{\tilde{R}_{\theta}})\bar{L}_fL_{\pi}\bar{L}_{\tilde{f}_{\theta}}^T\cdot\frac{(E+\tilde{E}+2\sigma_{\zeta}\sqrt{d})\tilde{E}}{\sigma_{\zeta}^2}\right] \\
&=12T^4(L_R+L_{\tilde{R}_{\theta}})\bar{L}_fL_{\pi}\bar{L}_{\tilde{f}_{\theta}}^Td,
\end{align*}
where we have used the fact that $\mathbb{E}_{p(\vec{\zeta})}[E]=T^{-1}\sum_{t=0}^{T-1}\mathbb{E}_{p(\zeta_t)}[\|\zeta_t\|]\le\sigma_{\zeta}\sqrt{d}$, and similarly $\mathbb{E}_{p_{\xi}(\xi)}[\tilde{E}]=T^{-1}\sum_{t=0}^{T-1}\mathbb{E}_{p_{\xi}(\xi)}[\|\xi_t\|]\le\sigma_{\zeta}\sqrt{d}$. Therefore, we have
\begin{align*}
\|\hat{D}_{\text{PG}}(\theta;\zeta)-\nabla_{\theta}J(\theta)\|
\le\|\hat{D}_{\text{PG}}(\theta;\zeta)\|+\|\nabla_{\theta}J(\theta)\|
\le3T^4(L_R+L_{\tilde{R}_{\theta}})\bar{L}_fL_{\pi}\bar{L}_{\tilde{f}_{\theta}}^Td\cdot\left(4d+\frac{(\tilde{E}+E+2\sigma_{\zeta}\sqrt{d})\tilde{E}}{\sigma_{\zeta}^2}\right),
\end{align*}
as claimed.

\paragraph{Upper bound on the sample complexity of $\hat{D}_{\text{PG}}-\nabla_{\theta}J$.}

We have $E'=(\tilde{E}+E+2\sigma_{\zeta}\sqrt{d})\tilde{E}\le\|\phi\|_1$, where we think of $\phi$ as the $T^2(d_A+d_S+1)d_A$ values $\xi_{t,i}\xi_{t',i'}$, $\zeta_{t,j}\xi_{t',i'}$, and $2\sigma_{\zeta}\sqrt{d}\xi_{t',i'}$, for all $t,t'\in\{0,1,...,T-1\}$, $i,i'\in[d_A]$, and $j\in[d_S]$. Since $\xi_t$ and $\zeta_t$ are $\sigma_{\zeta}$-sub-Gaussian for each $t\in T$, by Lemma~\ref{lem:subgaussianprodbound}, $\phi$ is $(\tau,b)$-sub-exponential, where $\tau,b=O(d\sigma_{\zeta}^2)$. Thus, we can apply Lemma~\ref{lem:subexponentialbound} with
\begin{align*}
Y&=\hat{D}_{\text{PG}}(\theta;\zeta)-\nabla_{\theta}J(\theta) \\
X&=E' \\
A&=\frac{3T^4(L_R+L_{\tilde{R}_{\theta}})\bar{L}_fL_{\pi}\bar{L}_{\tilde{f}_{\theta}}^Td}{\sigma_{\zeta}^2} \\
B&=0.
\end{align*}
Thus, $Y$ is $(\tau_{\text{PG}},b_{\text{PG}})$-sub-exponential, where
\begin{align*}
\tau_{\text{PG}},b_{\text{PG}}
=O(A(\tau+b)d\log d+B)
=O\left(T^6(L_R+L_{\tilde{R}_{\theta}})\bar{L}_fL_{\pi}\bar{L}_{\tilde{f}_{\theta}}^Td^4\log(Td)\right).
\end{align*}
Thus, by Lemma~\ref{lem:hoeffdingvector}, the sample complexity of $\hat{D}_{\text{PG}}(\theta)-\nabla_{\theta}J(\theta)$ is
\begin{align*}
\sqrt{n_{\text{PG}}(\epsilon,\delta)}&=\frac{\tau_{\text{PG}}\sqrt{2\log(2Td_A/\delta)}}{\epsilon} \\
&=O\left(\frac{T^6(L_R+L_{\tilde{R}_{\theta}})\bar{L}_fL_{\pi}\bar{L}_{\tilde{f}_{\theta}}^Td^4\log(T)\log(d)^{3/2}\log(1/\delta)^{1/2}}{\epsilon}\right),
\end{align*}
for all $\epsilon\le d\tau_{\text{PG}}^2/b_{\text{PG}}$. The claim follows.

\paragraph{Lower bound on the sample complexity of $\hat{D}_{\text{PG}}-\nabla_{\theta}J$.}

Consider a linear dynamical system with $S=A=\mathbb{R}$, time-varying deterministic transitions
\begin{align*}
f_t(s,a)&=\begin{cases}\beta(s+a)&\text{if}~s=0\\\beta s&\text{otherwise},\end{cases}
\end{align*}
zero noise $p_t(\zeta)=\delta(0)$ (i.e., $\sigma_{\zeta}=0$), initial state $s_0=0$, time-varying rewards
\begin{align*}
R_t(s,a)=\begin{cases}s&\text{if}~t=T-1\\0&\text{otherwise},\end{cases}
\end{align*}
control policy class $\pi_{\theta}(s)=\theta$, current parameters $\theta=0$, and action noise $p_{\xi}$. Note that
\begin{align*}
a_t=\theta+\tau_{\xi}\xi_t,
\end{align*}
where $\xi_t\sim p_{\xi}(\xi)$ i.i.d., so
\begin{align*}
s_t=\begin{cases}0&\text{if}~t=0\\\beta^{t-1}(\theta+\tau_{\xi}\xi)&\text{otherwise},\end{cases}
\end{align*}
where $\xi=\xi_0$ is the action noise on the first step. Note that
\begin{align*}
\hat{Q}_{\theta}^{(t)}(\xi)&=\beta^{T-2}(\theta+\tau_{\xi}\xi),
\end{align*}
and
\begin{align*}
\tilde{V}_{\theta}^{(t)}(s)&=\begin{cases}
\mathbb{E}_{p_{\xi}(\xi)}[\beta^{T-2}(s+\theta+\tau_{\xi}\xi)]=0&\text{if}~t=0 \\
\beta^{T-t-2}s&\text{otherwise}.
\end{cases}
\end{align*}
In particular, note that
\begin{align*}
\hat{Q}_{\theta}^{(t)}(\xi)-\tilde{V}_{\theta}^{(t)}(s_t)=\begin{cases}
\beta^{T-2}(\theta+\tau_{\xi}\xi)&\text{if}~t=0 \\
0&\text{otherwise}.
\end{cases}
\end{align*}
Also, note that $\nabla_{\theta}J(\theta)=\beta^{T-2}$. Therefore, we have
\begin{align*}
\nabla_{\theta}\log\tilde{\pi}(a\mid s)=\nabla_{\theta}\log p_{\xi}\left(\frac{a-\theta}{\tau_{\xi}}\right)=-\frac{\nabla_{\xi}p_{\xi}\left(\frac{a-\theta}{\tau_{\xi}}\right)}{\tau_{\xi}\cdot p_{\xi}\left(\frac{a-\theta}{\tau_{\xi}}\right)}=-\frac{1}{\tau_{\xi}}\cdot\nabla_{\xi}\log p_{\xi}\left(\frac{a-\theta}{\tau_{\xi}}\right).
\end{align*}
Thus, for i.i.d. samples $\xi^{(1)},...,\xi^{(n)}\sim p_{\xi}(\xi)$, we have
\begin{align*}
\hat{D}_{\text{PG}}(0)-\nabla_{\theta}J(0)
&=\frac{1}{n}\sum_{i=1}^n\left(\hat{Q}_{\theta}^{(t)}(\xi^{(i)})-\tilde{V}_{\theta}^{(t)}(s_t^{(i)})\right)\cdot\left(-\nabla_{\theta}\log\tilde{\pi}(a_t^{(i)}\mid s_t^{(i)})\right)-\beta^{T-2} \\
&=\frac{1}{n}\sum_{i=1}^n\beta^{T-2}\tau_{\xi}\xi^{(i)}\cdot\left(-\frac{1}{\tau_{\xi}}\cdot\nabla_{\xi}\log p_{\xi}(\xi^{(i)})\right)-\beta^{T-2} \\
&=-\beta^{T-2}\left[1+\frac{1}{n}\sum_{i=1}^n\xi^{(i)}\cdot\nabla_{\xi}\log p_{\xi}(\xi^{(i)})\right].
\end{align*}
Note that for $p_{\xi}(\xi)$ satisfying our conditions (differentiable on $\mathbb{R}$ and satisfying $\lim_{\xi\to\pm\infty}\xi\cdot p_{\xi}(\xi)=0$), we have
\begin{align}
\label{eqn:integrationbyparts}
\mathbb{E}_{p_{\xi}(\xi)}[\xi\cdot\nabla_{\xi}\log p_{\xi}(\xi)]&=\int_{-\infty}^{\infty}\xi\cdot\nabla_{\xi}p_{\xi}(\xi)d\xi=-\int_{-\infty}^{\infty}p_{\xi}(\xi)d\xi=-1,
\end{align}
where the second-to-last step follows from integration by parts. Thus, by the definition of the sample complexity,
\begin{align*}
\text{Pr}\left[\left|\frac{1}{n}\sum_{i=1}^n\xi^{(i)}\cdot\nabla_{\xi}\log p_{\xi}(\xi^{(i)})+1\right|\ge\epsilon\right]>\delta
\end{align*}
for any $n<n_{\xi}(\epsilon,\delta)$, so we have
\begin{align*}
\text{Pr}\left[|\hat{D}_{\text{PG}}(0)-\nabla_{\theta}J(0)|\ge\epsilon\right]
=\text{Pr}\left[\beta^{T-2}\left|\frac{1}{n}\sum_{i=1}^n\xi^{(i)}\cdot\nabla_{\xi}\log p_{\xi}(\xi^{(i)})+1\right|\ge\beta^{T-2}\epsilon\right]>\delta.
\end{align*}
for any $n<n_{\xi}(\epsilon/\beta^{T-2},\delta)$. Thus, we have
\begin{align*}
n_{\text{PG}}(\epsilon,\delta)\ge n_{\xi}(\epsilon/\beta^{T-2},\delta).
\end{align*}
Next, consider the case where $p_{\xi}(\xi)=\mathcal{N}(\xi\mid0,\sigma^2)$, for any $\sigma\in\mathbb{R}_+$. Then, we have
\begin{align*}
\nabla_{\xi}\log p_{\xi}(\xi)=\nabla_{\xi}\left(-\log\sqrt{2\pi}-\frac{\|\xi\|^2}{2\sigma^2}\right)=-\frac{\xi}{\sigma^2},
\end{align*}
so
\begin{align*}
\hat{D}_{\text{PG}}(0)-\nabla_{\theta}J(0)
=\beta^{T-2}\left[-1+\frac{1}{n\sigma^2}\sum_{i=1}^n(\xi^{(i)})^2\right]
=\beta^{T-2}\left[-1+\frac{1}{n}\sum_{i=1}^n(x^{(i)})^2\right],
\end{align*}
where $x^{(i)}\sim\mathcal{N}(0,1)$ are i.i.d. standard Gaussian random variables for $i\in[n]$. By Lemma~\ref{lem:chisquaredlower}, letting $x=n^{-1}\sum_{i=1}^n(x^{(i)})^2$ (so $\mu_x=\mathbb{E}_{p(x)}=1$), for
\begin{align*}
n\le\min\left\{\frac{2\beta^{T-2}\left(\frac{1}{2}\log(1/\delta)+\log(1/e^2\sqrt{2})\right)}{\epsilon},\frac{1}{\delta}\right\},
\end{align*}
we have
\begin{align*}
\text{Pr}\left[\hat{D}_{\text{PG}}(0)-\nabla_{\theta}J(0)\ge\epsilon\right]
=\text{Pr}_{p(x)}\left[x\ge\mu_x+\frac{\epsilon}{\beta^{T-2}}\right]
\ge\frac{1}{\sqrt{n}}\cdot\frac{1}{e^2\sqrt{2}}e^{-\frac{n\epsilon}{2\beta^{T-2}}}
\ge\sqrt{\delta}\cdot\sqrt{\delta}
=\delta.
\end{align*}
Thus, the sample complexity of $\hat{D}_{\text{PG}}-\nabla_{\theta}J(\theta)$ satisfies
\begin{align*}
n_{\text{PG}}(\epsilon,\delta)\ge\min\left\{\frac{2\beta^{T-2}\left(\frac{1}{2}\log(1/\delta)+\log(1/e^2\sqrt{2})\right)}{\epsilon},\frac{1}{\delta}\right\}.
\end{align*}
Note that the numerator is positive as long as $\delta\le1/12$. The claim follows, as does the theorem statement. $\qed$

\section{Proof of Theorem~\ref{thm:fdupper}}

\paragraph{Preliminaries.}

Note that the expected cumulative reward is equivalent to
\begin{align*}
J(\theta)&=V_{\theta}^{(0)}(s_0) \\
V_{\theta}^{(t)}(s)&=R_{\theta}(s)+\mathbb{E}_{p(\zeta)}\left[V_{\theta}^{(t+1)}(f_{\theta}(s)+\zeta)\right]\hspace{0.2in}(\forall t\in\{0,1,...,T-1\}) \\
V_{\theta}^{(T)}(s)&=0.
\end{align*}
Similarly, given a sample $\vec{\zeta}\sim p(\vec{\zeta})$, the stochastic approximation of the expected cumulative reward is
\begin{align*}
\hat{J}(\theta;\vec{\zeta})&=\hat{V}_{\theta}^{(0)}(s_0;\vec{\zeta}) \\
\hat{V}_{\theta}^{(t)}(s;\vec{\zeta})&=R_{\theta}(s)+\hat{V}_{\theta}^{(t+1)}(f_{\theta}(s)+\zeta_t;\vec{\zeta})\hspace{0.2in}(\forall t\in\{0,1,...,T-1\}) \\
\hat{V}_{\theta}^{(T)}(s;\vec{\zeta})&=0.
\end{align*}
The finite difference approximation of $\nabla_{\theta}J(\theta)$ is
\begin{align*}
D_{\text{FD}}(\theta)=\sum_{k=1}^{d_{\Theta}}\frac{J(\theta+\lambda\nu^{(k)})-J(\theta-\lambda\nu^{(k)})}{2\lambda}\cdot\nu^{(k)},
\end{align*}
where $\nu^{(k)}$ is a basis vector for $k\in[d]$ and $d_{\Theta}$ is the dimension of the parameter space $\Theta=\mathbb{R}^d$. Finally, an estimate of the finite difference approximation for two samples $\zeta,\eta\sim\tilde{p}(\zeta)$ is
\begin{align*}
\hat{D}_{\text{FD}}(\theta;\vec{\zeta},\vec{\eta})=\sum_{k=1}^{d_{\Theta}}\frac{\hat{J}(\theta+\lambda\nu^{(k)};\vec{\zeta})-\hat{J}(\theta-\lambda\nu^{(k)};\vec{\eta})}{2\lambda}\cdot\nu^{(k)},
\end{align*}
where $\hat{J}(\theta;\vec{\zeta})$ is as defined in the proof of Theorem~\ref{thm:mbupper}.

\paragraph{Bounding the deviation of $\hat{V}_{\theta}^{(t)}$ from $V_{\theta}^{(t)}$.}

We claim that for $t\in\{0,1,...,T\}$, we have
\begin{align*}
\|\hat{V}_{\theta}^{(t)}(s;\vec{\zeta})-V_{\theta}^{(t)}(s)\|&\le B^{(t)}(\vec{\zeta})
\end{align*}
for all $\theta\in\Theta$ and $s\in S$, where
\begin{align*}
B^{(t)}(\vec{\zeta})&=\sum_{i=t}^{T-1}L_V^{(i+1)}(\|\zeta_i\|+\sigma_{\zeta}\sqrt{d_A}),
\end{align*}
where $L_V^{(t)}$ is a Lipschitz constant for $V_{\theta}^{(t)}$. The base case $t=T$ follows trivially. Note that $\sigma_{\zeta}\sqrt{d_A}\ge\sqrt{\mathbb{E}_{p(\zeta)}[\|\zeta\|^2]}\ge\mathbb{E}_{p(\zeta)}[\|\zeta\|]$. Then, for $t\in\{0,1,...,T-1\}$, we have
\begin{align*}
\|\hat{V}_{\theta}^{(t)}(s;\vec{\zeta})-V_{\theta}^{(t)}(s)\|
=&\left\|\hat{V}_{\theta}^{(t+1)}(f_{\theta}(s)+\zeta_t;\vec{\zeta})-\mathbb{E}_{p(\zeta)}\left[V_{\theta}^{(t+1)}(f_{\theta}(s)+\zeta)\right]\right\| \\
\le&\|\hat{V}_{\theta}^{(t+1)}(f_{\theta}(s)+\zeta_t;\vec{\zeta})-V_{\theta}^{(t+1)}(f_{\theta}(s)+\zeta_t)\| \\
&+\mathbb{E}_{p(\zeta)}\left[\|V_{\theta}^{(t+1)}(f_{\theta}(s)+\zeta_t)-V_{\theta}^{(t+1)}(f_{\theta}(s)+\zeta)\|\right] \\
\le&B^{(t+1)}(\vec{\zeta})+L_V^{(t+1)}(\|\zeta_t\|+\sigma_{\zeta}\sqrt{d_A}) \\
=&B^{(t)}(\vec{\zeta}).
\end{align*}
The claim follows.

\paragraph{Bounding the deviation of $\hat{D}_{\text{FD}}$ from $D_{\text{FD}}$.}

Let
\begin{align*}
D_{\text{FD}}(\theta)\mathbb{E}_{p(\vec{\zeta}),p(\vec{\eta})}[\hat{D}_{\text{FD}}(\theta)].
\end{align*}
Then, letting $L_{\nabla V}=\operatorname*{\arg\max}_{t\in\{0,1,...,T\}}L_{\nabla V}^{(t)}$, note that
\begin{align*}
\|\hat{J}(\theta;\vec{\zeta})-J(\theta)\|\le B^{(0)}(\vec{\zeta})=\sum_{i=0}^{T-1}L_V^{(i+1)}(\|\zeta_i\|+\sigma_{\zeta}\sqrt{d_A})\le3T^3L_{R_{\theta}}\bar{L}_{f_{\theta}}^T(E+\sigma_{\zeta}\sqrt{d_A}),
\end{align*}
where $E=T^{-1}\sum_{t=0}^{T-1}\|\zeta_t\|$. Thus, we have
\begin{align*}
\|\hat{D}_{\text{FD}}(\theta;\zeta,\eta)_k-D_{\text{FD}}(\theta)_k\|
&=\left\|\frac{\hat{J}(\theta+\lambda\nu^{(k)};\vec{\zeta})-\hat{J}(\theta-\lambda\nu^{(k)};\vec{\eta})}{2\lambda}\cdot\nu^{(k)}-\frac{J(\theta+\lambda\nu^{(k)})-J(\theta-\lambda\nu^{(k)})}{2\lambda}\cdot\nu^{(k)}\right\| \\
&\le\frac{\|\hat{J}(\theta+\lambda\nu^{(k)};\vec{\zeta})-J(\theta+\lambda\nu^{(k)})\|+\|\hat{J}(\theta-\lambda\nu^{(k)};\vec{\eta})-J(\theta-\lambda\nu^{(k)})\|}{2\lambda} \\
&\le\frac{3T^3L_{R_{\theta}}\bar{L}_{f_{\theta}}^T(E+\tilde{E}+2\sigma_{\zeta}\sqrt{d_A})}{2\lambda}
\end{align*}
for $k\in[d_{\Theta}]$, where $\tilde{E}=T^{-1}\sum_{t=0}^{T-1}\|\eta_t\|$.

\paragraph{Upper bound on the sample complexity of $\hat{D}_{\text{FD}}-D_{\text{FD}}$.}

Note that $E+\tilde{E}\le\|E'\|_1$, where $E'=\vec{\zeta}\circ\vec{\eta}$ is the length $2Td_S$ concatenation of the vectors $\zeta_0,\zeta_1,...,\zeta_{T-1},\eta_0,\eta_1,...,\eta_{T-1}$, so $E'$ is $\sigma_{\zeta}$-sub-Gaussian. We apply Lemma~\ref{lem:subgaussianbound} with
\begin{align*}
Y&=\hat{D}_{\text{FD}}(\theta;\vec{\zeta},\vec{\eta})_k-D_{\text{FD}}(\theta)_k \\
X&=E' \\
A&=\frac{3T^3L_{R_{\theta}}\bar{L}_{f_{\theta}}^T}{\lambda} \\
B&=A\sigma_{\zeta}\sqrt{d_A}.
\end{align*}
Thus, $Y$ is $\sigma_{\text{FD}}$-sub-Gaussian, where
\begin{align*}
\sigma_{\text{FD}}
&=\max\{10A\sigma(2Td_A)\log(2Td_A),5A\sigma_{\zeta}\sqrt{d_A})\} \\
&=20A\sigma_{\zeta}Td_A\log(Td_A) \\
&\le\frac{60T^4L_{R_{\theta}}\bar{L}_{f_{\theta}}^T\sigma_{\zeta}d_A\log(Td_A)}{\lambda}.
\end{align*}
Thus, by Lemma~\ref{lem:hoeffdingvector}, for $k\in[d_{\Theta}]$, the sample complexity of $\hat{D}_{\text{FD}}(\theta)_k-D_{\text{FD}}(\theta)_k$ is
\begin{align*}
\sqrt{\tilde{n}_{\text{FD}}(\tilde{\epsilon},\tilde{\delta})}&=\frac{\sigma_{\text{FD}}\sqrt{2\log(2d_A/\tilde{\delta})}}{\tilde{\epsilon}} \\
&=O\left(\frac{T^4L_{R_{\theta}}\bar{L}_{f_{\theta}}^T\sigma_{\zeta}d_A\log(T)\log(d_A)^{3/2}\log(1/\tilde{\delta})^{1/2}}{\lambda\tilde{\epsilon}}\right).
\end{align*}

\paragraph{Upper bound on the sample complexity of $\hat{D}_{\text{FD}}-\nabla_{\theta}J(\theta)$.}

By Theorem~\ref{thm:fd}, we have
\begin{align*}
\nabla_{\theta}J(\theta)&=D_{\text{FD}}(\theta)+\Delta,
\end{align*}
where
\begin{align*}
\|\Delta\|\le L_{\nabla J}d_A\lambda\le44T^5\bar{L}_{R_{\theta}}\bar{L}_{f_{\theta}}^{4T}d_A\lambda,
\end{align*}
where the second inequality follows from the fact that $L_{\nabla J}=L_{\nabla V}^{(0)}$ and the bound on $L_{\nabla V}^{(0)}$ in Lemma~\ref{lem:nablavlipschitz}. Now, taking
\begin{align*}
\lambda&=\frac{\epsilon}{88T^5\bar{L}_{R_{\theta}}\bar{L}_{f_{\theta}}^{4T}d_A} \\
\tilde{\epsilon}&=\frac{\epsilon}{2\sqrt{d_{\Theta}}} \\
\tilde{\delta}&=\frac{\delta}{d_{\Theta}},
\end{align*}
then with probability $1-\delta$, we have
\begin{align*}
\|\hat{D}_{\text{FD}}(\theta)-\nabla_{\theta}J(\theta)\|\le\|\hat{D}_{\text{FD}}(\theta)-D_{\text{FD}}(\theta)\|+\|\Delta\|\le\epsilon,
\end{align*}
so the sample complexity of $\hat{D}_{\text{FD}}(\theta)-\nabla_{\theta}J(\theta)$ is
\begin{align*}
\sqrt{n_{\text{FD}}(\epsilon,\delta)}=
&=O\left(\frac{T^9\bar{L}_{R_{\theta}}^2\bar{L}_{f_{\theta}}^{5T}\sigma_{\zeta}d_A^2\sqrt{d_{\Theta}}\log(T)\log(d_A)^{3/2}\log(d_{\Theta})^{1/2}\log(1/\tilde{\delta})^{1/2}}{\epsilon^2}\right).
\end{align*}
The claim follows.

\paragraph{Lower bound on the sample complexity of $\hat{D}_{\text{FD}}-\nabla_{\theta}J(\theta)$.}

Consider a linear dynamical system with $S=\mathbb{R}^2$, $A=\mathbb{R}$, time-varying deterministic transitions
\begin{align*}
f_t((s,s'),a)&=\begin{cases}\beta(s,s'+a)&\text{if}~s=0\\\beta(s,s')&\text{otherwise},\end{cases}
\end{align*}
time-varying noise
\begin{align*}
p_t((\zeta,0))=\begin{cases}\mathcal{N}(\zeta\mid 0,\sigma_{\zeta}^2)&\text{if}~t=0\\\delta(0)&\text{otherwise},\end{cases}
\end{align*}
where $\sigma_{\zeta}\in\mathbb{R}$, initial state $s_0=(0,0)$, time-varying rewards
\begin{align*}
R_t((s,s'),a)=\begin{cases}s+\phi(s')&\text{if}~t=T-1\\0&\text{otherwise},\end{cases}
\end{align*}
where $\phi:\mathbb{R}\to\mathbb{R}$ is defined by
\begin{align*}
\phi(x)=\begin{cases}2x-1&\text{if}~x\ge1\\x^2&\text{if}~-1\le x<1\\2x+1&\text{if}~x<-1,\end{cases}
\end{align*}
control policy class $\pi_{\theta}((s,s'))=\theta$, and current parameters $\theta=0$. Note that technically, $R$ is not twice continuously differentiable, so it does not satisfy Assumption~\ref{assump:lipschitz}. However, the only place in the proof of Theorem~\ref{thm:fdupper} where we need this assumption is to apply Lemma~\ref{lem:lipschitzderivative} in Lemma~\ref{lem:nablavlipschitz}. By the discussion in the proof of Lemma~\ref{lem:lipschitzderivative}, the lemma still applies, so our theorems still apply to this dynamical system. Now, we have
\begin{align*}
s_t=\begin{cases}0&\text{if}~t=0\\\beta^{t-1}(\zeta,\theta)&\text{otherwise},\end{cases}
\end{align*}
where $\zeta=\zeta_0$ is the noise on the first step. Thus, we have
\begin{align*}
\hat{J}(\theta;\zeta)=s_{T-1}+s_{T-1}'=\beta^{T-2}\zeta+\phi(\beta^{T-2}\theta).
\end{align*}
Also, note that
\begin{align*}
\nabla_{\theta}J(\theta)=\mathbb{E}_{p(\zeta)}[\nabla_{\theta}\hat{J}(0;\zeta)]=\phi'(\beta^{T-2}\theta)\cdot\beta^{T-2},
\end{align*}
so $\nabla_{\theta}J(0)=0$, since $\phi'(0)=0$.

Next, note that for $2n$ i.i.d. samples $\zeta^{(1)},...,\zeta^{(n)},\eta^{(1)},...,\eta^{(n)}\sim\mathcal{N}(0,\sigma_{\zeta}^2)$, we have
\begin{align*}
\hat{D}_{\text{FD}}(0)-\nabla_{\theta}J(0)
&=\frac{1}{2\lambda}\left[\frac{1}{n}\sum_{i=1}^n\hat{J}(\lambda;\zeta^{(i)})-\frac{1}{n}\sum_{i=1}^n\hat{J}(-\lambda;\eta^{(i)})\right] \\
&=\frac{1}{2\lambda}\cdot\frac{1}{n}\sum_{i=1}^n\left[\beta^{T-2}\zeta^{(i)}-\beta^{T-2}\eta^{(i)}\right]+\frac{1}{2\lambda}\left[\phi(\beta^{T-2}\lambda)-\phi(-\beta^{T-2}\lambda)\right].
\end{align*}
Letting $\zeta^{(n+i)}=-\eta^{(i)}$ for $i\in[n]$, and using the fact that $\phi(-x)=-\phi(x)$, we have
\begin{align*}
\hat{D}_{\text{FD}}(0)-\nabla_{\theta}J(0)
=\frac{1}{2\lambda n}\sum_{i=1}^{2n}\beta^{T-2}\zeta^{(i)}+\frac{1}{\lambda}\cdot\phi(\beta^{T-2}\lambda)
\sim\mathcal{N}\left(\mu_{\text{FD}},\frac{\sigma_{\text{FD}}}{n}\right).
\end{align*}
where
\begin{align*}
\mu_{\text{FD}}&=\phi(\beta^{T-2}\lambda) \\
\sigma_{\text{FD}}&=\frac{\beta^{T-2}\sigma_{\zeta}}{\lambda}.
\end{align*}
Thus, by Lemma~\ref{lem:gaussianlower}, for
\begin{align*}
n\le\frac{\sigma_{\text{FD}}^2\left(\log\left(\sqrt{\frac{e}{2\pi}}\right)+\log(1/\tilde{\delta})\right)}{\epsilon^2},
\end{align*}
and recalling that $D_{\text{FD}}(\theta)=\mathbb{E}_{p_{\theta}(\alpha)}[\hat{D}_{\text{FD}}(\theta;\alpha)]=\mu_{\text{FD}}$, we have
\begin{align*}
\text{Pr}\left[\hat{D}_{\text{FD}}(0)-D_{\text{FD}}(0)\ge\tilde{\epsilon}\right]=\text{Pr}_{x\sim\mathcal{N}(0,\sigma_{\text{FD}}^2/n)}[|x|\ge\tilde{\epsilon}]\ge\sqrt{\frac{e}{2\pi}}\cdot e^{-n\epsilon^2/\sigma_{\text{FD}}^2}\ge\tilde{\delta}.
\end{align*}
Thus, the sample complexity of $\hat{D}_{\text{FD}}(0)-D_{\text{FD}}(0)$ satisfies
\begin{align*}
\tilde{n}_{\text{FD}}(\tilde{\epsilon},\tilde{\delta})\ge\frac{\sigma_{\text{FD}}^2\left(\log\left(\sqrt{\frac{e}{2\pi}}\right)+\log(1/\tilde{\delta})\right)}{\tilde{\epsilon}^2}.
\end{align*}
Now, recall that $\nabla_{\theta}J(0)=0$, so
\begin{align*}
\text{Pr}\left[\hat{D}_{\text{FD}}(0)-\nabla_{\theta}J(0)\ge\epsilon\right]
=\text{Pr}\left[\hat{D}_{\text{FD}}(0)\ge\epsilon\right]
=\text{Pr}\left[\hat{D}_{\text{FD}}(0)-D_{\text{FD}}(0)\ge\epsilon-\mu_{\text{FD}}\right].
\end{align*}
Thus, using our assumption $\delta\le1/2$, then we need to have $\mu_{\text{FD}}\le\epsilon$ for $\text{Pr}\left[\hat{D}_{\text{FD}}(0)-\nabla_{\theta}J(0)\ge\epsilon\right]\le\delta$ to hold. As a consequence, using our assumption $\epsilon\le1$, we must have
\begin{align*}
\epsilon\ge\mu_{\text{FD}}=\phi(\beta^{T-2}\lambda)=\beta^{2(T-2)}\lambda^2,
\end{align*}
where the last step follows since $0\le\phi(\beta^{T-2}\lambda)\le1$ implies $\phi(x)=x^2$. Thus, we have $\lambda\le\sqrt{\frac{\epsilon}{\beta^{2(T-2)}}}$, so we have $\sigma_{\text{FD}}\ge\beta^{4(T-2)}\sigma_{\zeta}^2/\epsilon$. Finally, we have
\begin{align*}
\text{Pr}\left[\hat{D}_{\text{FD}}(0)-\nabla_{\theta}J(0))\ge\epsilon\right]\ge\text{Pr}\left[\hat{D}_{\text{FD}}(0)-D_{\text{FD}}(0)\ge\epsilon\right],
\end{align*}
so the sample complexity of $\hat{D}_{\text{FD}}(0)-\nabla_{\theta}J(\theta)$ satisfies
\begin{align*}
n_{\text{FD}}(\epsilon,\delta)
\ge\tilde{n}_{\text{FD}}(\epsilon,\delta)
&\ge\frac{\sigma_{\text{FD}}^2(T-2)^2\beta^{2(T-3)}\cdot\left(\log\left(\sqrt{\frac{e}{2\pi}}\right)+\log(1/\delta)\right)}{\epsilon^2} \\
&\ge\frac{(T-2)^2\beta^{6(T-3)}\sigma_{\zeta}^2\cdot\left(\log\left(\log(1/\delta)+\sqrt{\frac{e}{2\pi}}\right)\right)}{\epsilon^4}.
\end{align*}
Finally, for any $d_{\Theta}\in\mathbb{N}$, we can consider $d_{\Theta}$ independent copies of this dynamical system. Then, estimating the gradient $\nabla_{\theta}J(\theta)$ is equivalent to estimating $\frac{dJ}{d\theta_i}(\theta)$ for each $i\in[d_{\Theta}]$. Thus, we have
\begin{align*}
n_{\text{FD}}(\epsilon,\delta)
\ge\tilde{n}_{\text{FD}}(\epsilon,\delta)
&\ge\frac{(T-2)^2\beta^{6(T-3)}\sigma_{\zeta}^2d_{\Theta}\cdot\left(\log\left(\log(1/\delta)+\sqrt{\frac{e}{2\pi}}\right)\right)}{\epsilon^4}.
\end{align*}
The claim follows, as does the theorem statement. $\qed$

\section{Bounds on Lipschitz Constants}

We prove bounds on the Lipschitz constants $L_V^{(t)}$ for $V_{\theta}^{(t)}$, $L_{\nabla V}^{(t)}$ for $\nabla V_{\theta}^{(t)}$, and $L_{\tilde{V}}^{(t)}$ for $\tilde{V}_{\theta}^{(t)}$. We use implicitly use the commonly known results in Appendix~\ref{sec:lipschitzappendix} throughout these proofs.

\begin{lemma}
\label{lem:vlipschitz}
We claim that for $t\in\{0,1,...,T\}$, $V_{\theta}^{(t)}$ is $L_V^{(t)}$-Lipschitz, where
\begin{align*}
L_V^{(t)}\le3T^2L_{R_{\theta}}\bar{L}_{f_{\theta}}^{T-t-1}.
\end{align*}
\end{lemma}

\begin{proof}
First, we show that $V_{\theta}^{(t)}$ is $L_{V,\theta}^{(t)}$-Lipschitz in $\theta$ and $L_{V,s}^{(t)}$-Lipschitz in $s$, where
\begin{align*}
L_{V,\theta}^{(t)}&=\sum_{i=t}^{T-1}(L_{R_{\theta}}+L_{f_{\theta}}L_{V,s}^{(i+1)}) \\
L_{V,s}^{(t)}&=\sum_{i=t}^{T-1}L_{f_{\theta}}^{i-t}L_{R_{\theta}},
\end{align*}
We prove by induction. The base case $t=T$ is trivial. Then, for $t\in\{0,1,...,T-1\}$, note that $V_{\theta}^{(t)}$ is $(L_{V,\theta}^{(t)})'$-Lipschitz in $\theta$, where
\begin{align*}
(L_{V,\theta}^{(t)})'=L_{R_{\theta}}+L_{V,\theta}^{(t+1)}+L_{f_{\theta}}L_{V,s}^{(t+1)}=L_{V,\theta}^{(t)}.
\end{align*}
Similarly, note that $V_{\theta}^{(t)}$ is $(L_{V,s}^{(t)})'$-Lipschitz in $s$, where
\begin{align*}
(L_{V,s}^{(t)})'=L_{R_{\theta}}+L_{f_{\theta}}L_{V,s}^{(t+1)}=L_{V,s}^{(t)},
\end{align*}
as was to be shown. Finally, note that
\begin{align*}
L_{V,s}^{(t)}\le TL_{R_{\theta}}\bar{L}_{f_{\theta}}^{T-t-1},
\end{align*}
so
\begin{align*}
L_{V,\theta}^{(t)}\le T(L_{R_{\theta}}+L_{f_{\theta}}\cdot TL_{R_{\theta}}\bar{L}_{f_{\theta}}^{T-t-2})\le2T^2L_{R_{\theta}}\bar{L}_{f_{\theta}}^{T-t-1}.
\end{align*}
Thus, $V_{\theta}^{(T)}$ is $(L_V^{(t)})'$-Lipschitz, where
\begin{align*}
(L_V^{(t)})\le L_{V,\theta}^{(t)}+L_{V,s}^{(t)}\le3T^2L_{R_{\theta}}\bar{L}_{f_{\theta}}^{T-t-1}=L_V^{(t)}.
\end{align*}
The claim follows.
\end{proof}

\begin{lemma}
\label{lem:nablavlipschitz}
We claim that for $t\in\{0,1,...,T\}$, $\nabla V_{\theta}^{(t)}$ is $L_{\nabla V}^{(t)}$-Lipschitz, where
\begin{align*}
L_{\nabla V}^{(t)}=44T^5\bar{L}_{R_{\theta}}\bar{L}_{f_{\theta}}^{4(T-t-1)}.
\end{align*}
\end{lemma}

\begin{proof}
First, we show that $\nabla_{\theta}V_{\theta}^{(t)}$ is $L_{\nabla V,\theta,\theta}^{(t)}$-Lipschitz in $\theta$ and $L_{\nabla V,\theta,s}^{(t)}$-Lipschitz in $s$, and that $\nabla_sV_{\theta}^{(t)}$ is $L_{\nabla V,\theta,s}^{(t)}$-Lipschitz in $\theta$ and $L_{\nabla V,s,s}^{(t)}$-Lipschitz in $s$, where
\begin{align*}
L_{\nabla V,\theta,\theta}^{(t)}&=\sum_{i=t}^{T-1}(L_{\nabla R_{\theta}}+2L_{f_{\theta}}L_{\nabla V,\theta,s}^{(i+1)}+L_{f_{\theta}}^2L_{\nabla V,s,s}^{(i+1)}+L_{\nabla f_{\theta}}L_V^{(i+1)}) \\
L_{\nabla V,\theta,s}^{(t)}&=\sum_{i=t}^{T-1}L_{f_{\theta}}^{i-t}(L_{\nabla R_{\theta}}+L_{f_{\theta}}^2L_{\nabla V,s,s}^{(i+1)}+L_{\nabla f_{\theta}}L_V^{(i+1)}) \\
L_{\nabla V,s,s}^{(t)}&=\sum_{i=t}^{T-1}L_{f_{\theta}}^{2(i-t)}(L_{\nabla R_{\theta}}+L_{\nabla f_{\theta}}L_V^{(i+1)}) \\
L_{\nabla V,\theta,\theta}^{(T)}&=L_{\nabla V,\theta,s}^{(T)}=L_{\nabla V,s,s}^{(T)}=0.
\end{align*}
We prove by induction. The base case $t=T$ is trivial. First, for $t\in\{0,1,...,T-1\}$, note that $\nabla_{\theta}V_{\theta}^{(t)}$ is $(L_{\nabla V,\theta,\theta}^{(t)})'$-Lipschitz in $\theta$, where
\begin{align*}
(L_{\nabla V,\theta,\theta}^{(t)})'=L_{\nabla R_{\theta}}+L_{\nabla V,\theta,\theta}^{(t+1)}+L_{f_{\theta}}L_{\nabla V,\theta,s}^{(t+1)}+L_{f_{\theta}}(L_{\nabla V,\theta,s}^{(t+1)}+L_{f_{\theta}}L_{\nabla V,s,s}^{(t+1)})+L_{\nabla f_{\theta}}L_V^{(t+1)}=L_{\nabla V,\theta,\theta}^{(t)}.
\end{align*}
Second, note that $\nabla_{\theta}V_{\theta}^{(t)}$ is $(L_{\nabla V,\theta,s}^{(t)})'$-Lipschitz in $s$, where
\begin{align*}
(L_{\nabla V,\theta,s}^{(t)})'=L_{\nabla R_{\theta}}+L_{f_{\theta}}L_{\nabla V,\theta,s}^{(t+1)}+L_{f_{\theta}}^2L_{\nabla V,s,s}^{(t+1)}+L_{\nabla f_{\theta}}L_V^{(t+1)}=L_{\nabla V,\theta,s}^{(t)}.
\end{align*}
Third, note that $\nabla_sV_{\theta}^{(t)}$ is $(L_{\nabla V,s,\theta}^{(t)})'$-Lipschitz in $\theta$, where
\begin{align*}
(L_{\nabla V,s,\theta}^{(t)})'=L_{\nabla R_{\theta}}+L_{f_{\theta}}(L_{\nabla V,\theta,s}^{(t+1)}+L_{f_{\theta}}L_{\nabla V,s,s}^{(t+1)})+L_{\nabla f_{\theta}}L_V^{(t+1)}=L_{\nabla V,\theta,s}^{(t)}.
\end{align*}
Fourth, note that $\nabla_sV_{\theta}^{(t)}$ is $(L_{\nabla V,s,s}^{(t)})'$-Lipschitz in $s$, where
\begin{align*}
(L_{\nabla V,s,s}^{(t)})'=L_{\nabla R_{\theta}}+L_{f_{\theta}}^2L_{\nabla V,s,s}^{(t+1)}+L_{\nabla f_{\theta}}L_V^{(t+1)}=L_{\nabla V,s,s}^{(t)},
\end{align*}
as was to be shown. Finally, note that
\begin{align*}
L_{\nabla V,s,s}^{(t)}\le T\bar{L}_{f_{\theta}}^{2(T-t-1)}(L_{\nabla R_{\theta}}+L_{\nabla f_{\theta}}\cdot3T^2L_{R_{\theta}}\bar{L}_{f_{\theta}}^{T-t-2})\le 4T^3\bar{L}_{R_{\theta}}\bar{L}_{f_{\theta}}^{3(T-t-1)},
\end{align*}
so
\begin{align*}
L_{\nabla V,\theta,s}^{(t)}\le T\bar{L}_{f_{\theta}}^{T-t-1}(L_{\nabla R_{\theta}}+L_{f_{\theta}}^2\cdot4T^3\bar{L}_{R_{\theta}}\bar{L}_{f_{\theta}}^{3(T-t-2)}+L_{\nabla f_{\theta}}\cdot3T^2L_{R_{\theta}}\bar{L}_{f_{\theta}}^{T-t-2})\le8T^4\bar{L}_{R_{\theta}}\bar{L}_{f_{\theta}}^{4(T-t-1)}
\end{align*}
so
\begin{align*}
L_{\nabla V,\theta,\theta}^{(t)}&\le T(L_{\nabla R_{\theta}}+2L_{f_{\theta}}\cdot8T^4\bar{L}_{R_{\theta}}\bar{L}_{f_{\theta}}^{4(T-t-2)}+L_{f_{\theta}}^2\cdot4T^3\bar{L}_{R_{\theta}}\bar{L}_{f_{\theta}}^{3(T-t-2)}+L_{\nabla f_{\theta}}\cdot 3T^2L_{R_{\theta}}\bar{L}_{f_{\theta}}^{T-t-2}) \\
&\le24T^5\bar{L}_{R_{\theta}}\bar{L}_{f_{\theta}}^{4(T-t-1)}.
\end{align*}
Thus, $\nabla V_{\theta}^{(t)}$ is $(L_{\nabla V}^{(t)})'$-Lipschitz, where
\begin{align*}
(L_{\nabla V}^{(t)})'=L_{\nabla V,\theta,\theta}+2L_{\nabla V,\theta,s}+L_{\nabla V,s,s}\le44T^5\bar{L}_{R_{\theta}}\bar{L}_{f_{\theta}}^{4(T-t-1)}=L_{\nabla V}^{(t)}.
\end{align*}
The claim follows.
\end{proof}

\begin{lemma}
\label{lem:tildevlipschitz}
We claim that for $t\in\{0,1,...,T\}$, $\tilde{V}_{\theta}^{(t)}$ is $L_{\tilde{V}}^{(t)}$-Lipschitz, where
\begin{align*}
L_{\tilde{V}}^{(t)}=3T^2L_{\tilde{R}_{\theta}}\bar{L}_{\tilde{f}_{\theta}}^{T-t-1}.
\end{align*}
\end{lemma}

\begin{proof}
Note that $\tilde{V}_{\theta}^{(t)}$ is exactly equal to $V_{\theta}^{(t)}$ with $R_{\theta}$ replaced with $\tilde{R}_{\theta}$ and $f_{\theta}$ replaced with $\tilde{f}_{\theta}$. Thus, the claim follows by the same argument as for Lemma~\ref{lem:vlipschitz}.
\end{proof}

\section{Proof of Theorem~\ref{thm:fd}}
\label{sec:fdproof}

\begin{theorem}
(Taylor's theorem) Let $f:\mathbb{R}\to\mathbb{R}$ be an everywhere differentiable function with $L_{f'}$-Lipschitz derivative. Then, for any $x,\epsilon\in\mathbb{R}$, we have
\begin{align*}
f(x+\epsilon)=f(x)+f'(x)\cdot\epsilon+\Delta,
\end{align*}
where
\begin{align*}
|\Delta|\le\frac{L_{f'}\epsilon^2}{2}.
\end{align*}
\end{theorem}

\begin{proof}
The claim follows from Theorem 5.15 in~\cite{rudin1976principles}, together with Lemma~\ref{lem:lipschitzderivative}, which implies that $|f''(x)|\le L_{f'}$ for all $x\in\mathbb{R}$.
\end{proof}

Now, we prove Theorem~\ref{thm:fd}. By Taylor's theorem, we have
\begin{align*}
f(x+\mu)=f(x)+\langle\nabla f(x),\mu\rangle+\Delta(\mu),
\end{align*}
where
\begin{align*}
\|\Delta(\mu)\|\le\frac{1}{2}L_{\nabla f}\|\mu\|^2.
\end{align*}
Thus, we have
\begin{align*}
&\sum_{k=1}^d\frac{f(x+\lambda\nu^{(k)})-f(x-\lambda\nu^{(k)})}{2\lambda}\cdot\nu^{(k)} \\
&=\sum_{k=1}^d\frac{(f(x)+\langle\nabla f(x),\lambda\nu^{(k)}\rangle+\Delta(\lambda\nu^{(k)}))-(f(x)-\langle\nabla f(x),\lambda\nu^{(k)}\rangle+\Delta(-\lambda\nu^{(k)}))}{2\lambda}\cdot\nu^{(k)} \\
&=\sum_{k=1}^d\langle\nabla f(x),\nu^{(k)}\rangle\cdot\nu^{(k)}+\frac{\Delta(\lambda\nu^{(k)})-\Delta(-\lambda\nu^{(k)})}{2\lambda}\cdot\nu^{(k)} \\
&=\sum_{k=1}^d\nu^{(k)}((\nu^{(k)})^{\top}\nabla f(x))+\sum_{k=1}^d\frac{\Delta(\lambda\nu^{(k)})-\Delta(-\lambda\nu^{(k)})}{2\lambda}\cdot\nu^{(k)} \\
&=\nabla f(x)+\sum_{k=1}^d\frac{\Delta(\lambda\nu^{(k)})-\Delta(-\lambda\nu^{(k)})}{2}\cdot\nu^{(k)}
\end{align*}
Therefore, we have
\begin{align*}
\Delta=\sum_{k=1}^d\frac{\Delta(\lambda\nu^{(k)})-\Delta(-\lambda\nu^{(k)})}{2\lambda}\cdot\nu^{(k)},
\end{align*}
so
\begin{align*}
\|\Delta\|
\le\sum_{k=1}^d\left\|\frac{\Delta(\lambda\nu^{(k)})-\Delta(-\lambda\nu^{(k)})}{2\lambda}\cdot\nu^{(k)}\right\|
\le\frac{1}{2}L_{\nabla f}\lambda\cdot\|\nu^{(k)}\|^3
\le L_{\nabla f}d\lambda,
\end{align*}
as claimed. $\qed$

\section{Technical Lemmas (Lipschitz Constants)}
\label{sec:lipschitzappendix}

We define Lipschitz continuity (for the $L_2$ norm), and prove a number of standard results.

\begin{definition}
\rm
A function $f:\mathcal{X}\to\mathcal{Y}$ (where $\mathcal{X}\subseteq\mathbb{R}^d$ and $\mathcal{Y}\subseteq\mathbb{R}^{d'}$) is \emph{$L_f$-Lipschitz continuous} if for all $x,x'\in\mathcal{X}$,
\begin{align}
\label{eqn:lipschitz}
\|f(x)-f(x')\|\le L_f\|x-x'\|.
\end{align}
\end{definition}
If $\mathcal{X}$ is a space of matrices or tensors, we assume $x$ and $x'$ are unrolled into vectors. in (\ref{eqn:lipschitz}).

\begin{lemma}
\label{lem:lipschitzderivative}
If $f:\mathcal{X}\to\mathcal{Y}$ is $L_f$-Lipschitz and continuously differentiable, then for all $x\in\mathcal{X}$,
\begin{align*}
\|\nabla f(x)\|\le L_f.
\end{align*}
\end{lemma}

\begin{proof}
Note that
\begin{align*}
\nabla f(x)=\lim_{\|\epsilon\|\to0}\frac{f(x+\epsilon)-f(x)}{\|\epsilon\|},
\end{align*}
so
\begin{align*}
\|\nabla f(x)\|=\lim_{\|\epsilon\|\to0}\frac{\|f(x+\epsilon)-f(x)\|}{\|\epsilon\|}\le\lim_{\|\epsilon\|\to0}\frac{L_f\|\epsilon\|}{\|\epsilon\|}=L_f,
\end{align*}
as claimed. Note that the result holds even if each component $f_i$ is continuously differentiable except on a finite set $X$. In particular, for each point $x\in X$, we can use the standard definition $(\nabla f(x))_i=(f_{i,+}'(x)+f_{i,-}'(x))/2$, where $f_{i,+}'(x)$ is the right derivative and $f_{i,-}'(x)$ is the left deriviative. Letting $(\nabla_+f(x))_i=f_{i,+}'(x)$ and $(\nabla_-f(x))_i=f_{i,-}'(x)$, then $\nabla f(x)=(\nabla_+f(x)+\nabla_-f(x))/2$. Then, we have
\begin{align*}
\|\nabla f(x)\|\le\frac{\|\nabla_+f(x)\|+\|\nabla_-f(x)\|}{2}\le L_f,
\end{align*}
as claimed.
\end{proof}

\begin{lemma}
If $f,g:\mathcal{X}\to\mathcal{Y}$ are $L_f$- and $L_g$-Lipschitz, respectively, then $h(x)=f(x)+g(x)$ is $L_h$-Lipschitz, where $L_h=L_f+L_g$.
\end{lemma}

\begin{proof}
Note that
\begin{align*}
\|h(x)-h(x')\|\le\|f(x)-f(x')\|+\|g(x)-g(x')\|\le(L_f+L_g)\|x-x'\|=L_h\|x-x'\|,
\end{align*}
as claimed.
\end{proof}

\begin{lemma}
If $f,g:\mathcal{X}\to\mathcal{Y}$ where $f$ is $L_f$-Lipschitz and bounded by $M_f$ (i.e., $|f(x)|\le M_f$ for all $x\in\mathcal{X}$), and $g$ is $L_g$-Lipschitz and bounded by $M_g$. Then $h(x)=f(x)\cdot g(x)$ is $L_h$-Lipschitz, where $L_h=M_gL_f+M_fL_g$.
\end{lemma}

\begin{proof}
Note that
\begin{align*}
\|h(x)-h(x')\|&\le\|(f(x)-f(x'))g(x)\|+\|(g(x)-g(x'))f(x')\| \\
&\le M_gL_f\|x-x'\|+M_fL_g\|x-x'\| \\
&=L_h\|x-x'\|,
\end{align*}
as claimed.
\end{proof}

\begin{lemma}
If $f:\mathcal{X}\to\mathcal{Y}$ is $L_f$-Lipschitz and $g:\mathcal{Y}\to\mathcal{Z}$ is $L_g$-Lipschitz, then $h(x)=g(f(x))$ is $L_h$-Lipschitz, where $L_h=L_gL_f$.
\end{lemma}

\begin{proof}
Note that
\begin{align*}
\|g(f(x))-g(f(x'))\|\le L_g\|f(x)-f(x')\|\le L_gL_f\|x-x'\|\le L_h\|x-x'\|,
\end{align*}
as claimed.
\end{proof}

\begin{lemma}
Let $f:\mathcal{X}\times\mathcal{Y}\to\mathcal{Z}$ be $L_{f,x}$-Lipschitz in $\mathcal{X}$ (for all $y\in\mathcal{Y}$) and $L_{f,y}$-Lipschitz in $\mathcal{Y}$ (for all $x\in\mathcal{X}$). Then, $f$ is $L_f$-Lipschitz in $\mathcal{X}\times\mathcal{Y}$, where $L_f=L_{f,x}+L_{f,y}$.
\end{lemma}

\begin{proof}
Note that
\begin{align*}
\|f(x,y)-f(x',y')\|
&\le\|f(x,y)-f(x',y)\|+\|f(x',y)-f(x',y')\| \\
&\le L_{f,x}\|x-x'\|+L_{f,y}\|y-y'\| \\
&\le L_{f,x}\|(x,y)-(x',y')\|+L_{f,y}\|(x,y)-(x',y')\| \\
&\le(L_{f,x}+L_{f,y})\|(x,y)-(x',y')\| \\
&=L_f\|(x,y)-(x',y')\|,
\end{align*}
as claimed.
\end{proof}

\begin{lemma}
Let $f:\mathcal{X}\to\mathcal{Y}$ be $L_f$-Lipschitz, and $g:\mathcal{X}\to\mathcal{Z}$ be $L_g$-Lipchitz. Then, $h(x)=(f(x),g(x))$ is $L_h$-Lipschitz, where $L_h=L_f+L_g$.
\end{lemma}

\begin{proof}
Note that
\begin{align*}
\|h(x)-h(x')\|
&\le\|(f(x)-f(x'),g(x)-g(x'))\| \\
&=\sqrt{\sum_{i=1}^{d_{\mathcal{Y}}}(f_i(x)-f_i(x'))^2+\sum_{j=1}^{d_{\mathcal{Z}}}(g_i(x)-g_i(x'))^2} \\
&\le\sqrt{\sum_{i=1}^{d_{\mathcal{Y}}}(f_i(x)-f_i(x'))^2}+\sqrt{\sum_{j=1}^{d_{\mathcal{Z}}}(g_i(x)-g_i(x'))^2} \\
&=\|f(x)-f(x')\|+\|g(x)-g(x')\| \\
&\le L_f\|x-x'\|+L_g\|x-x'\| \\
&\le (L_f+L_g)\|x-x'\| \\
&=L_h\|x-x'\|,
\end{align*}
as claimed.
\end{proof}

\begin{lemma}
Let $f:\mathcal{X}\times\mathcal{Z}\to\mathcal{Y}$ be $L_f$-Lipschitz. Then, $g(x)=\mathbb{E}_{p(z)}[f(x,z)]$ (where $p(z)$ is a distribution over $\mathcal{Z}$) is $L_g$-Lipschitz, where $L_g=L_f$.
\end{lemma}

\begin{proof}
Note that
\begin{align*}
\|g(x)-g(x')\|\le\mathbb{E}_{p(z)}\left[\|f(x,z)-f(x',z)\|\right]\le L_f\|x-x'\|=L_g\|x-x'\|,
\end{align*}
as claimed.
\end{proof}

\section{Technical Lemmas (Sub-Gaussian Random Variables)}
\label{sec:subgaussianappendix}

We define sub-Gaussian random variables, and prove a number of standard results. We also prove Lemma~\ref{lem:subgaussianbound}, a key lemma that enables us to infer a sub-Gaussian constant for a random variable bounded $Y$ in norm by a sub-Gaussian random variable $X$, i.e., $\|Y\|\le A\|X\|_1+B$ (where $\|\cdot\|$ is the $L_2$ norm). This lemma is a key step in the proofs of our upper bounds for the model-based and finite-difference policy gradient estimators. Finally, we also prove Lemma~\ref{lem:gaussianlower}, which is a key step in the proof of our lower bounds.

\begin{definition}
\rm
A random variable $X$ over $\mathbb{R}$ is \emph{$\sigma_X$-sub-Gaussian} if $\mathbb{E}[X]=0$, and for all $t\in\mathbb{R}$, we have $\mathbb{E}[e^{tX}]\le e^{\sigma_X^2t^2/2}$.
\end{definition}

\begin{lemma}
If a random variable $X$ over $\mathbb{R}$ is $\sigma_X$-sub-Gaussian, then $\mathbb{E}[|X|^2]\le\sigma_X^2$.
\end{lemma}

\begin{proof}
See~\cite{stromberg1994probability}.
\end{proof}

\begin{lemma}
\label{lem:hoeffding}
(Hoeffding's inequality)
Let $x_1,...,x_n\sim p_X(x)$ be i.i.d. $\sigma_X$-sub-Gaussian random variables over $\mathbb{R}$. Then,
\begin{align*}
\text{Pr}\left[\left|\frac{1}{n}\sum_{i=1}^nx_n\right|\ge\epsilon\right]\le2e^{-\frac{n\epsilon^2}{2\sigma_X^2}}.
\end{align*}
\end{lemma}

\begin{proof}
See Proposition 2.1 of~\cite{wainwright2019high}.
\end{proof}

\begin{definition}
\rm
A random vector $X$ over $\mathbb{R}^d$ is $\sigma_X$-sub-Gaussian if each $X_i$ is $\sigma_X$-sub-Gaussian.
\end{definition}

\begin{lemma}
If a random vector $X$ over $\mathbb{R}^d$ is $\sigma_X$-sub-Gaussian, then $\mathbb{E}[\|X\|]\le\sigma_X\sqrt{d}$.
\end{lemma}

\begin{proof}
Note that
\begin{align*}
\mathbb{E}[\|X\|]=\mathbb{E}\left[\sqrt{\sum_{i=1}^d\|X_i\|^2}\right]\le\sqrt{\sum_{i=1}^d\mathbb{E}[\|X_i\|^2]}\le\sigma_X\sqrt{d},
\end{align*}
where the first inequality follows from Jensen's inequality.
\end{proof}

\begin{lemma}
\label{lem:hoeffdingvector}
Let $X$ be random vector over $\mathbb{R}^d$ with mean $\mu_X=\mathbb{E}[X]$, such that $X-\mu_X$ is $\sigma_X$-sub-Gaussian. Then, given $\epsilon,\delta\in\mathbb{R}_+$, the sample complexity of $X$ satisfies
\begin{align*}
n_X(\epsilon,\delta)\le\frac{2\sigma_X^2\log(2d/\delta)}{\epsilon^2},
\end{align*}
i.e., given $x_1,...,x_n\sim p_X(x)$ i.i.d. samples of $X$ with empirical mean $x=n^{-1}\sum_{i=1}^nx_n$, then $\text{Pr}[\|x-\mu_X\|\ge\epsilon]\le\delta$.
\end{lemma}

\begin{proof}
Note that
\begin{align*}
\text{Pr}[\|x-\mu_X\|\ge\epsilon]\le\text{Pr}[\|x-\mu_X\|_1\ge\epsilon]\le\sum_{i=1}^d\text{Pr}\left[|x_i-\mu_{X,i}|\ge\frac{\epsilon}{d}\right]\le2de^{-\frac{nt^2}{2\sigma_X^2}}\le\delta,
\end{align*}
as claimed.
\end{proof}

\begin{lemma}
\label{lem:subgaussianbound}
Let $X$ be a $\sigma_X$-sub-Gaussian random vector over $\mathbb{R}^d$, and let $Y$ be a random vector over $\mathbb{R}^{d'}$ satisfying
\begin{align*}
\|Y\|\le A\|X\|_1+B,
\end{align*}
where $A,B\in\mathbb{R}_+$. Then $Y$ is $\sigma_Y$-sub-Gaussian, where
\begin{align*}
\sigma_Y=\max\{10A\sigma_Xd\log d,5B\}.
\end{align*}
\end{lemma}

\begin{proof}
We first prove that $|Y_i|$ is bounded for each $i\in[d]$, and then use this fact to prove that $Y_i$ is sub-Gaussian. In particular, we claim that for any $i\in[d]$ and any $t\in\mathbb{R}_+$, we have
\begin{align*}
\text{Pr}[|Y_i|\ge t]\le2e^{-\frac{t^2}{2\tilde{\sigma}_Y^2}},
\end{align*}
where
\begin{align*}
\tilde{\sigma}_Y=\max\left\{4A\sigma_Xd\sqrt{\log d},2B\right\}.
\end{align*}
To this end, note that by Theorem 5.1 in~\cite{lattimore2018bandit}, for any $i\in[d]$ and any $t\in\mathbb{R}_+$, we have
\begin{align*}
\text{Pr}[|X_i|\ge t]\le2e^{-\frac{t^2}{2\sigma_X^2}}.
\end{align*}
Now, note that
\begin{align*}
\text{Pr}[|Y_i|\ge t]&\le\text{Pr}[\|Y\|\ge t]\le\text{Pr}\left[\|X\|_1\ge\frac{t-B}{A}\right]\le\sum_{i=1}^d\text{Pr}\left[|X_i|\ge\frac{t-B}{Ad}\right]\le2de^{-\frac{(t-B)^2}{(Ad\sigma_X\sqrt{2})^2}}.
\end{align*}
We consider three cases. First, suppose that $t\ge\max\{4A\sigma_Xd\sqrt{\log d},2B\}$. Then, $(t-B)^2\ge(t/2)^2$, so
\begin{align*}
\text{Pr}[|Y_i|\ge t]\le2de^{-\frac{t^2}{(Ad\sigma_X\sqrt{8})^2}}=2e^{-\frac{t^2-(Ad\sigma_X\sqrt{8})^2\log d}{(Ad\sigma_X\sqrt{8})^2}}.
\end{align*}
Furthermore, $t^2-(Ad\sigma_X\sqrt{8})^2\log d\ge(t^2/2)$, so
\begin{align*}
\text{Pr}[|Y_i|\ge t]\le2e^{-\frac{t^2-(Ad\sigma_X\sqrt{8})^2\log d}{(Ad\sigma_X\sqrt{8})^2}}\le2e^{-\frac{t^2}{2(Ad\sigma_X\sqrt{8})^2}}\le2e^{-\frac{t^2}{2\tilde{\sigma}_Y^2}}.
\end{align*}
Second, if $t\le2B$, then
\begin{align*}
2e^{-\frac{t^2}{2\tilde{\sigma}_Y^2}}\ge2e^{-\frac{(2B)^2}{2\tilde{\sigma}_Y^2}}=2e^{-1/2}>1,
\end{align*}
so
\begin{align*}
\text{Pr}[|Y_i|\ge t]\le1\le2e^{-\frac{t^2}{2\tilde{\sigma}_Y^2}}.
\end{align*}
Third, if $t\le4A\sigma_Xd\sqrt{\log d}$, then
\begin{align*}
2e^{-\frac{t^2}{2\tilde{\sigma}_Y^2}}\ge2e^{-\frac{(4A\sigma_Xd\sqrt{\log d})^2}{2\tilde{\sigma}_Y^2}}\ge2e^{-1/2}>1,
\end{align*}
so
\begin{align*}
\text{Pr}[|Y_i|\ge t]\le1\le2e^{-\frac{t^2}{2\tilde{\sigma}_Y^2}}.
\end{align*}
As a consequence, by Note 5.4.2 in~\cite{lattimore2018bandit}, $Y_i$ is $\tilde{\sigma}_Y\sqrt{5}$-sub-Gaussian. Note that $\sigma_Y\ge\tilde{\sigma}_Y\sqrt{5}$, so the theorem follows.
\end{proof}

\begin{lemma}
\label{lem:gaussianlower}
Given $\sigma\in\mathbb{R}_+$,
\begin{align*}
\text{Pr}_{x\sim\mathcal{N}(0,\sigma^2)}[|x|\ge t]\ge\sqrt{\frac{e}{2\pi}}\cdot e^{-t^2/\sigma^2}.
\end{align*}
\end{lemma}

\begin{proof}
By Theorem 2 in~\cite{chang2011chernoff}, we have
\begin{align*}
1-\Phi(t)\ge\frac{1}{2}\sqrt{\frac{e}{2\pi}}\cdot e^{-t^2},
\end{align*}
where $\Phi(t)$ is the cumulative distribution function of $\mathcal{N}(0,1)$. Thus, for $\epsilon\in\mathbb{R}_+$, we have
\begin{align*}
\text{Pr}_{x\sim\mathcal{N}(0,\sigma^2)}[|x|\ge t]=\text{Pr}_{z\sim\mathcal{N}(0,1)}\left[|z|\ge\frac{t}{\sigma}\right]=2\left(1-\Phi\left(\frac{t}{\sigma}\right)\right)\ge\sqrt{\frac{e}{2\pi}}\cdot e^{-t^2/\sigma^2}\ge\delta.
\end{align*}
The claim follows.
\end{proof}

\section{Technical Lemmas (Sub-Exponential Random Variables)}
\label{sec:subexponentialappendix}

We define sub-exponential random variables, and prove a number of standard results. Additionally, we prove Lemma~\ref{lem:subexponentialbound} (an analog of Lemma~\ref{lem:subgaussianbound}), a key lemma that enables us to infer a sub-exponential constant for a random variable bounded $Y$ in norm by a sub-exponential random variable $X$, i.e., $\|Y\|\le A\|X\|_1+B$ (where $\|\cdot\|$ is the $L_2$ norm). This lemma is a key step in the proof of our upper bound in Theorem~\ref{thm:pgupper}. Finally, we also prove Lemma~\ref{lem:chisquaredlower}, which is a key step in the proof of our lower bound in Theorem~\ref{thm:pgupper}.

\begin{definition}
\rm
A random variable $X$ over $\mathbb{R}$ is \emph{$(\tau_X,b_X)$-sub-exponential} if $\mathbb{E}[X]=0$, and for all $t\in\mathbb{R}$ satisfying $|t|\le b_X^{-1}$, we have $\mathbb{E}[e^{tX}]\le e^{\tau_X^2t^2/2}$.
\end{definition}

\begin{lemma}
Let $x_1,...,x_n\sim p_X(x)$ be i.i.d. $(\tau_X,b_X)$-sub-exponential random variables over $\mathbb{R}$. Then, we have
\begin{align*}
\text{Pr}\left[\left|\frac{1}{n}\sum_{i=1}^nx_n\right|\ge\epsilon\right]\le\begin{cases}
2e^{-\frac{n\epsilon^2}{2\tau_X^2}}&\text{if}~|\epsilon|\le\tau_X^2/b_X \\
2e^{-\frac{n\epsilon}{2b_X}}&\text{otherwise}.
\end{cases}
\end{align*}
\end{lemma}

\begin{proof}
See (2.20) in~\cite{wainwright2019high}.
\end{proof}

\begin{definition}
\rm
A random vector $X$ over $\mathbb{R}^d$ is $(\tau_X,b_X)$-sub-exponential if each $X_i$ is $(\tau_X,b_X)$-sub-exponential.
\end{definition}

\begin{lemma}
\label{lem:subexponentialvector}
Let $X$ be a random vector over $\mathbb{R}^d$ with mean $\mu_X=\mathbb{E}[X]$, such that $X-\mu_X$ is $(\tau_X,b_X)$-sub-exponential. Then, given $\epsilon,\delta\in\mathbb{R}_+$ such that $\epsilon\le d\tau_X^2/b_X$, the sample complexity of $X$ satisfies
\begin{align*}
n_X(\epsilon,\delta)=\frac{2\tau_X^2\log(2d/\delta)}{\epsilon^2},
\end{align*}
i.e., given $x_1,...,x_n\sim p_X(x)$ i.i.d. samples of $X$ with empirical mean $x=n^{-1}\sum_{i=1}^nx_n$, then $\text{Pr}[\|x-\mu_X\|\ge\epsilon]\le\delta$.
\end{lemma}

\begin{proof}
Note that
\begin{align*}
\text{Pr}[\|x-\mu_X\|\ge\epsilon]\le\text{Pr}[\|x-\mu_X\|_1\ge\epsilon]\le\sum_{i=1}^d\text{Pr}\left[|x_i-\mu_{X,i}|\ge\frac{\epsilon}{d}\right]\le2de^{-\frac{nt^2}{2\tau_X^2}}\le\delta,
\end{align*}
as claimed.
\end{proof}

\begin{lemma}
\label{lem:subgaussiansquarebound}
Let $X$ be $\sigma_X$-sub-Gaussian. Then, $X^2$ is $(\tau_X,b_X)$-sub-exponential, where $\tau_X,b_X=O(\sigma_X^2)$.
\end{lemma}

\begin{proof}
The result follows from Lemma 5.5, Lemma 5.14, and the discussion preceding Definition 5.13 in~\cite{vershynin2010introduction}. In particular, using the notation in~\cite{vershynin2010introduction}, by Lemma 5.5, we have that $X$ satisfies $\|X\|_{\psi_2}=O(\sigma_X)$. Then, by Lemma 5.14, we have that $\|X^2\|_{\psi_1}=2\|X\|_{\psi_2}^2=O(\sigma_X^2)$. Finally, by the discussion preceding Definition 5.13, we have that $X^2$ is $(\tau_X,b_X)$-sub-exponential with parameters $\tau_X,b_X=O(\|X^2\|_{\psi_1})=O(\sigma_X^2)$. The claim follows.
\end{proof}

\begin{lemma}
\label{lem:subgaussianprodbound}
Let $X$ and $Y$ be $\sigma_X$-sub-Gaussian, respectively. Then, $Z=XY$ is $(\tau_Z,b_Z)$-sub-exponential, where $\tau_Z,b_Z=O(\sigma_X^2)$.
\end{lemma}

\begin{proof}
Note that
\begin{align*}
Z=XY=\frac{(X+Y)^2-(X-Y)^2}{4}.
\end{align*}
By Lemma~\ref{lem:subgaussiansquarebound}, we have $X+Y$ and $X-Y$ are $(\tau,b)$-sub-exponential for $\tau,b=O(\sigma_X^2)$, so $Z$ is $\tau_Z,b_Z$-sub-exponential, for $\tau_Z,b_Z=O(\tau+b)=O(\sigma_X^2)$, as claimed.
\end{proof}

\begin{lemma}
\label{lem:subexponentialbound}
Let $X$ be a $(\tau_X,b_X)$-sub-exponential random vector over $\mathbb{R}^d$, and let $Y$ be a random vector over $\mathbb{R}^{d'}$ satisfying
\begin{align*}
\|Y\|\le A\|X\|_1+B,
\end{align*}
where $A,B\in\mathbb{R}_+$. Then $Y$ is $(\tau_Y,b_Y)$-sub-exponential, where $\tau_Y,b_Y=O(A(\tau_X+b_X)d\log d+B)$.
\end{lemma}

\begin{proof}
We use Lemma 5.14 and the discussion preceding Definition 5.13 in~\cite{vershynin2010introduction}. In particular, let $\tilde{\tau}_X=\max\{\tau_X,b_X\}$; then, from the definition of sub-exponential random variables with $t=\tilde{\tau}_X^{-1}$, we have
\begin{align*}
\mathbb{E}\left[e^{\frac{X_i}{\tilde{\tau}}}\right]\le\mathbb{E}\left[e^{\frac{t^2}{2\tilde{\tau}_X^2}}\right]\le e
\end{align*}
for each $i\in[d]$. Thus, using the notation in~\cite{vershynin2010introduction}, so by the discussion preceding the Definition 5.13 in~\cite{vershynin2010introduction}, we have $X_i$ satisfies $\|X_i\|_{\psi_1}=O(\tilde{\tau}_X)$, and furthermore satisfies
\begin{align*}
\text{Pr}[|X_i|\ge t]\le3e^{-t/K}
\end{align*}
for all $t\in\mathbb{R}_+$, where $K=O(\|X_i\|_{\psi_1})=O(\tilde{\tau}_X)$. Thus, for each $i\in[d]$, we have
\begin{align*}
\text{Pr}[|Y_i|\ge t]
\le\text{Pr}\left[\|X\|_1\ge\frac{t-B}{A}\right]
\le\sum_{i=1}^d\text{Pr}\left[|X_i|\ge\frac{t-B}{Ad}\right]
\le de^{1-\frac{t-B}{AKd}}.
\end{align*}
Now, let
\begin{align*}
\tilde{\tau}_Y=\max\{4AKd\log d,2B\}.
\end{align*}
We consider three cases. First, suppose that $t\ge\max\{4AKd\log d,2B\}$. Then, $t-B\ge t/2$, so
\begin{align*}
\text{Pr}[|Y_i|\ge t]\le de^{1-\frac{t}{2AKd}}=e^{1-\frac{t-2AKd\log d}{2AKd}}.
\end{align*}
Furthermore, $t-2AKd\log d\ge t/2$, so
\begin{align*}
\text{Pr}[|Y_i|\ge t]\le e^{1-\frac{t-2AKd\log d}{2AKd}}\le e^{1-\frac{t}{4AKd}}\le e^{1-\frac{t}{\tilde{\tau}_Y}}.
\end{align*}
Second, if $t\le2B$, then
\begin{align*}
e^{1-\frac{t}{\tilde{\tau}_Y}}\ge e^{1-\frac{2B}{\tilde{\tau}_Y}}\ge1,
\end{align*}
so
\begin{align*}
\text{Pr}[|Y_i|\ge t]\le1\le e^{1-\frac{t}{\tilde{\tau}_Y}}.
\end{align*}
Third, if $t\le4AKd\log d$, then
\begin{align*}
e^{1-\frac{t}{\tilde{\tau}_Y}}\ge e^{1-\frac{4AKd\log d}{\tilde{\tau}_Y}}\ge1,
\end{align*}
so
\begin{align*}
\text{Pr}[|Y_i|\ge t]\le1\le e^{1-\frac{t}{\tilde{\tau}_Y}}.
\end{align*}
As a consequence, by the discussion preceding Definition 5.13 in~\cite{vershynin2010introduction}, we have $Y_i$ satisfies $\|Y_i\|_{\psi_1}=O(\tilde{\tau}_Y)$. Thus, by Lemma 5.15 in~\cite{vershynin2010introduction}, we have that $Y_i$ is $(\tau_Y,b_Y)$-sub-exponential, where
\begin{align*}
\tau_Y,b_Y=O(\|Y_i\|_{\psi_1})=O(\tilde{\tau}_Y)=O(AKd\log d+B)=O(A\tilde{\tau}_Xd\log d+B)=O(A(\tau_X+b_X)d\log d+B).
\end{align*}
The claim follows.
\end{proof}

\begin{lemma}
\label{lem:chisquaredlower}
Given $\sigma\in\mathbb{R}_+$, let
\begin{align*}
x=\frac{(x^{(1)})^2+...+(x^{(n)})^2}{n},
\end{align*}
where $x^{(1)},...,x^{(n)}\sim\mathcal{N}(0,\sigma^2)$ i.i.d., and let $\mu_x=\mathbb{E}_{p(x)}[x]=\sigma^2$. Then, we have
\begin{align*}
\text{Pr}_{p(x)}[x\ge\mu_x+\epsilon]\ge\frac{1}{e^2\sqrt{2n}}e^{-\frac{n\epsilon}{2\sigma^2}}.
\end{align*}
\end{lemma}

\begin{proof}
Let $z=(z^{(1)})^2+...+(z^{(n)})^2$ be the sum of the squares of $n$ i.i.d. standard Gaussian random variables $z^{(1)},...,z^{(n)}\sim\mathcal{N}(0,1)$. We assume that $n=2k$ is even. Then, $z$ is distributed according to the $\chi_{2k}^2$ distribution, which has density function
\begin{align*}
p_{2k}(z)=\frac{1}{2^k(k-1)!}z^{k-1}e^{-z},
\end{align*}
and mean $\mu_{2k}=2k$. For $z\ge\mu_{2k}=2k$, we have
\begin{align*}
p_{2k}(z)
\ge\frac{1}{2^k(k-1)!}(2k)^{k-1}e^{-z/2}
=\frac{1}{2}\cdot\frac{k^{k-1}}{(k-1)!}e^{-z/2}
\ge\frac{1}{2}\cdot\frac{k^{k-1}}{(k-1)^{k-1/2}e^{-k+2}}e^{-z/2}
\ge\frac{1}{2e^2\sqrt{k}}e^{k-z/2},
\end{align*}
where the second inequality follows from a result
\begin{align*}
n!\le n^{n+1/2}e^{1-n}
\end{align*}
based on Stirling's approximation~\cite{robbins1955remark}. Thus, for any $\epsilon\in\mathbb{R}_+$, we have
\begin{align*}
\text{Pr}_{z\sim\chi_{2k}^2}[z\ge\mu_{2k}+\epsilon]\ge\int_{\mu_{2k}+\epsilon}^{\infty}\frac{1}{2e^2\sqrt{k}}e^{k-z/2}=\frac{1}{2e^2\sqrt{k}}e^{k-(\mu_{2k}+\epsilon)/2}=\frac{1}{2e^2\sqrt{k}}e^{-\epsilon/2}.
\end{align*}
Finally, for $x=((x^{(1)})^2+...+(x^{(n)})^2)/n$, where $x^{(1)},...,x^{(n)}\sim\mathcal{N}(0,\sigma^2)$ i.i.d., note that $x=\frac{\sigma^2z}{n}$ and
\begin{align*}
\mu_x=\mathbb{E}_{p(x)}[x]=\frac{\sigma^2\mu_n}{n}=\sigma^2,
\end{align*}
so we have
\begin{align*}
\text{Pr}_{p(x)}[x\ge\mu_x+\epsilon]
=\text{Pr}_{z\sim\chi_n^2}\left[z\ge\mu_n+\frac{n\epsilon}{\sigma^2}\right]
\ge\frac{1}{e^2\sqrt{2n}}e^{-\frac{n\epsilon}{2\sigma^2}}.
\end{align*}
The claim follows.
\end{proof}

\section{Experimental Results}
\label{sec:expappendix}

We show enlarged versions of the plots from Figure~\ref{fig:exp}:
\begin{center}
\includegraphics[width=0.7\columnwidth]{mb.pdf} \\
Model-Based Algorithm \\ \vspace{0.05in}
\includegraphics[width=0.7\columnwidth]{fd.pdf} \\
Finite-Differences Algorithm \\\vspace{0.05in}
\includegraphics[width=0.7\columnwidth]{pg.pdf} \\
Policy Gradient Theorem Algorithm
\end{center}

\end{document}